\documentclass{article}

\usepackage{microtype}
\usepackage{graphicx}
\usepackage{subcaption}
\usepackage{booktabs} 

\usepackage{hyperref}
\usepackage{adjustbox}
\usepackage{booktabs}
\usepackage{makecell} 
\usepackage{pifont}
\newcommand{\xmark}{\ding{55}}

\usepackage[preprint]{icml2026}

\usepackage{amsmath}
\usepackage{amssymb}
\usepackage{mathtools}
\usepackage{amsthm}
\usepackage[table]{xcolor}

\usepackage[capitalize,noabbrev]{cleveref}

\theoremstyle{plain}
\newtheorem{theorem}{Theorem}[section]
\newtheorem{proposition}[theorem]{Proposition}

\theoremstyle{definition}

\newtheorem{assumption}[theorem]{Assumption}
\theoremstyle{remark}

\usepackage[textsize=tiny]{todonotes}

\icmltitlerunning{SEDiT: Mask-Free Video Subtitle Erasure via One-step Diffusion Transformer}

\begin{document}

\twocolumn[
  \icmltitle{SEDiT: Mask-Free Video Subtitle Erasure via One-step Diffusion Transformer}



  \icmlsetsymbol{equal}{*}

  \begin{icmlauthorlist}
    \icmlauthor{Zheng Hui}{comp}
    \icmlauthor{Yunlong Bai}{comp}
  \end{icmlauthorlist}

  \icmlaffiliation{comp}{Baidu Inc., Beijing, China}

  \icmlcorrespondingauthor{Zheng Hui}{zheng\_hui@aliyun.com}
  \icmlcorrespondingauthor{Yunlong Bai}{baiyunlong@baidu.com}

  \icmlkeywords{Machine Learning, ICML}

  \vskip 0.3in
]



\printAffiliationsAndNotice{}  

\begin{abstract}
Recent breakthroughs in video diffusion models have significantly accelerated the development of video editing techniques. However, existing methods often rely on inpainting video frames based on masked input, which requires extracting the target video mask in advance, and the precision of the segmentation directly affects the quality of the completion. In this paper, we present \textbf{SEDiT}, a novel one-stage video \textbf{S}ubtitle \textbf{E}rasure approach via One-step \textbf{Di}ffusion \textbf{T}ransformer. We introduce a mask-free inference approach that enables direct erasure of the targeted subtitle. The proposed one-stage framework mitigates the sub-optimality inherent in the two-stage processing of prior models.

Since subtitle removal is a localized editing task in which most pixels remain unchanged, the underlying distribution shift is minimal, making it well-suited to one-step generation under rectified flow. We empirically validate the reliability of one-step denoising and further provide a formal theoretical justification. Under the localized-editing structure of subtitle removal, the conditional optimal transport (OT) map and its induced rectified flow velocity field are \textbf{Lipschitz continuous} with respect to the latent variable, which underpins the theoretical feasibility of one-step sampling.

To address the challenge of long-term temporal consistency, we adopt a hybrid training strategy by occasionally conditioning the model with a clean first-frame latent. This facilitates temporal continuity, allowing each segment during inference to leverage the output of its predecessor. To avoid visible seams caused by cropping and reinserting processed targets, particularly in scenarios involving substantial motion, we feed the original video directly into SEDiT. Thanks to one-step and chunk-wise streaming inference, our method can efficiently handle native 1440p video with infinite length. Project page: \url{http://zheng222.github.io/SEDiT_project}
\end{abstract}

\section{Introduction}

In recent years, video diffusion models~\cite{stable-video-diffusion,hunyuanvideo,ltx-video,wan21} have made remarkable progress in video generation. Current mainstream models employ the Diffusion Transformer (DiT)~\cite{dit} architecture to achieve unprecedented generation quality. Video inpainting, a key technique in the field of video editing, has also benefited from advances in foundational video generation models, leading to a significant improvement in restoration quality.

Video object removal is designed to eliminate designated objects from video sequences while maintaining spatial background consistency and temporal coherence. An early GAN-based representative method, ProPainter~\cite{propainter}, addresses video inpainting by completing the optical flow within masked regions, combined with a mask-guided sparse Transformer to handle the video inpainting task. Due to limitations in model capacity and the inherent constraints of GAN-based generation, such approaches tend to produce visual artifacts in complex scenarios. Recently, diffusion-based video inpainting methods have gained popularity. DiffuEraser~\cite{diffueraser} utilizes the preliminary completion results of Propainter~\cite{propainter} combined with DDIM inversion to obtain better initialization. Minimax-Remover~\cite{minimax-remover} is built on the Wan2.1-1.3B~\cite{wan21} and replaces original text conditions with learnable contrastive tokens embedded into the self-attention stream. EraserDiT~\cite{eraserdit} employs LTX-Video~\cite{ltx-video} as a compact backbone and integrates a Vision-Language Model~\cite{qwen2.5-vl} for interactive object removal. Despite rapid advances, existing methods still face challenges including visual artifacts, blurring, temporal flickering, and a fundamental reliance on precise segmentation masks.

\begin{figure}[ht]
  \centering
  \begin{subfigure}[c]{0.45\textwidth}
    \centering
    \adjustbox{valign=c}{\includegraphics[width=\textwidth]{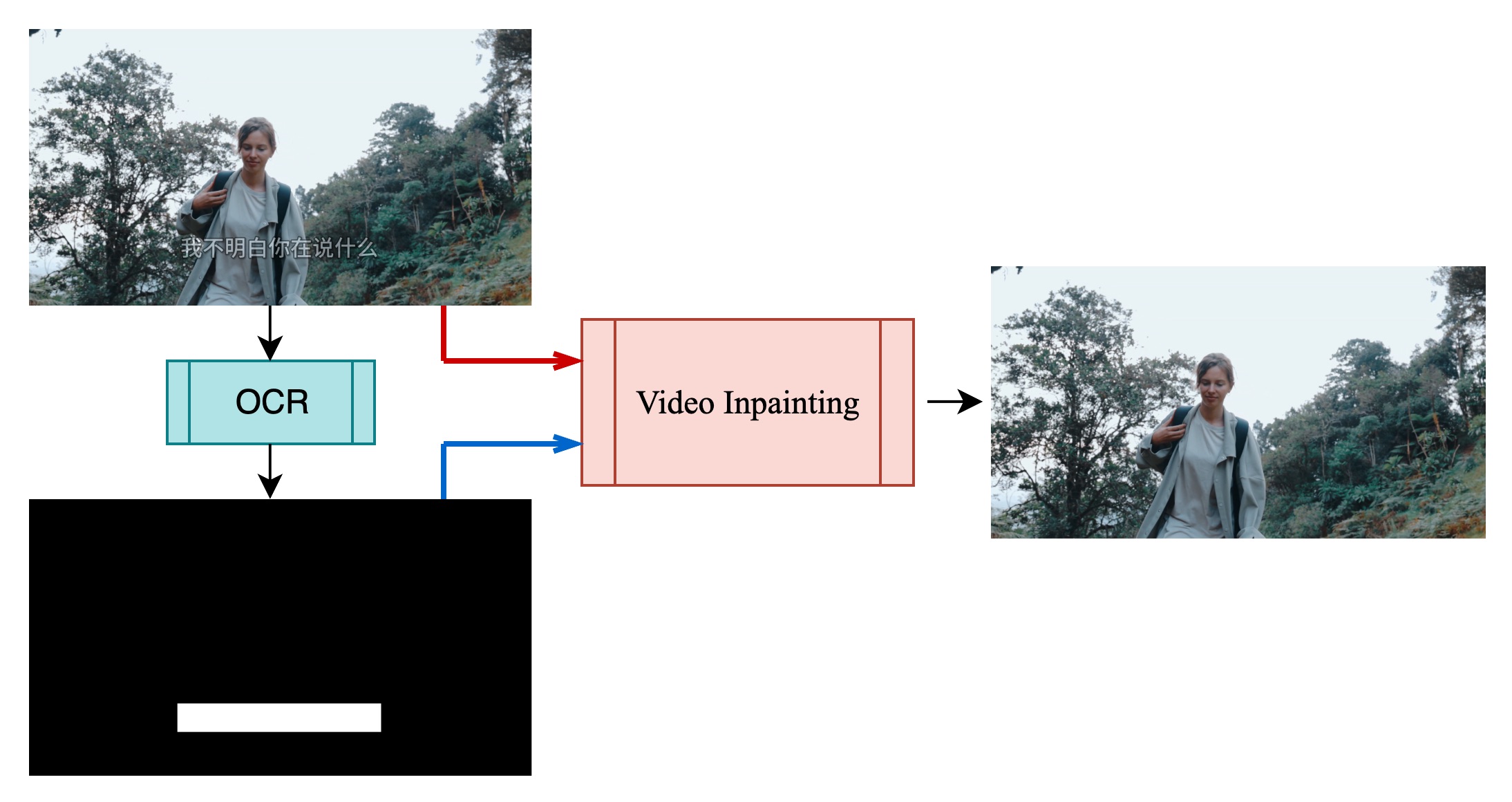}}
    \caption{Two-stage subtitle erasure}
    \label{fig:two-stage-framework}
  \end{subfigure}
  \begin{subfigure}[c]{0.45\textwidth}
    \centering
    \adjustbox{valign=c}{\includegraphics[width=\textwidth]{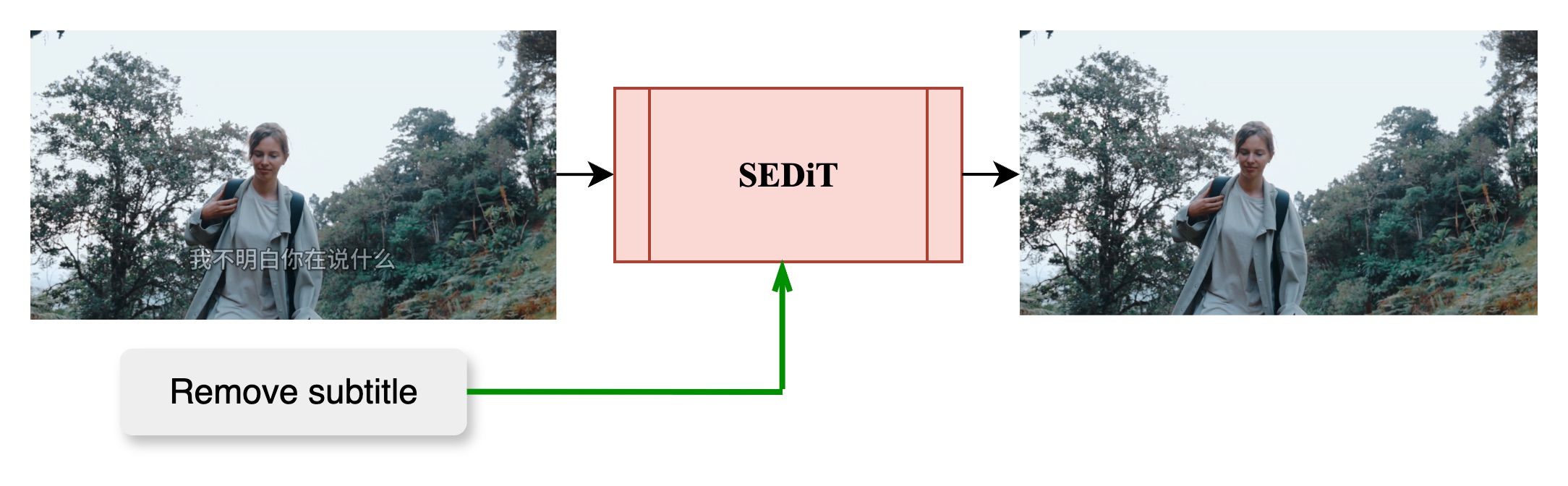}}
    \caption{One-stage subtitle erasure}
    \label{fig:end-to-end-framework}
  \end{subfigure}
  \caption{Comparison of frameworks between previous mask-based subtitle removal methods and our mask-free one-stage method.}
  \label{fig:framework-comparison}
\end{figure}

Subtitle removal in video represents a critical sub-task of object removal. Unlike generic object removal—which is typically interactive and processes short single-shot clips of around 10 seconds—subtitle removal must handle entire videos spanning multiple shots without any user interaction. Subtitles frequently exhibit complex visual properties such as transparency, stylized fonts, and gradient effects, which significantly hinder accurate segmentation. As illustrated in Figure~\ref{fig:two-stage-framework}, most prior methods use OCR-derived bounding boxes to localize subtitle regions. However, bounding boxes discard crucial pixel-level guidance information within character gaps, which severely limits restoration quality. To obtain fine-grained text masks, one would need a dedicated text matting model, substantially increasing the complexity of the overall system.

To address these limitations, we propose a one-stage framework for video subtitle removal that eliminates the need for OCR to localize subtitles. As illustrated in Figure~\ref{fig:end-to-end-framework}, we eliminate the OCR module and instead rely solely on the predefined prompt instruction to accomplish subtitle erasure. Inspired by recent advances in instruction-based image editing algorithms such as \emph{FLUX.1 Kontext}~\cite{flux-kontext} and \emph{Qwen-Image}~\cite{qwen-image}, we propose \emph{SEDiT}, a method that operates without explicit masks and performs the task using a fixed user prompt only. This simple yet effective design is particularly well-suited to video subtitle removal, a task that typically requires no user interaction and benefits from efficient batch processing. To enable direct processing of 1080p video, we adopt LTX-Video~\cite{ltx-video} as the backbone model, which incorporates a VAE with very high spatial-temporal compression. To fully exploit the capabilities of the base model while minimizing structural modifications, we adopt the same conditioning strategy as \emph{FLUX.1 Kontext}, where the conditional video is concatenated along the sequence dimension. In addition, to obtain paired training data, we designed a comprehensive data synthesis pipeline that includes font attribute configuration, subtitle orientation and positioning, as well as transition effects. This setup covers most subtitle scenarios and closely approximates the conditions found in real-world applications. To accommodate the demands of long-form video subtitle removal, we adopt a chunk-wise processing strategy. Except for the first chunk, each subsequent chunk leverages the last frame of the preceding chunk as a reference to enhance temporal consistency. Due to the strong conditioning capability of the reference video, satisfactory temporal consistency across chunks can be maintained by referencing only a single frame between adjacent segments.

Our main contributions can be summarized as follows.
\begin{itemize}
    \item We present SEDiT, a mask-free one-stage framework designed for video subtitle erasure. Unlike previous mask-based video inpainting approaches, our method eliminates the need for explicit mask generation to accomplish subtitle removal. We have developed a relatively comprehensive data synthesis pipeline that incorporates font attributes, subtitle rendering directions, and transition effects. Thanks to the high-quality synthetic data, our method achieves promising results in video subtitle removal.
  
  \item To inject the reference video without modifying the model architecture, we concatenate the conditional video latent and the noisy video latent along the sequence dimension. This extends the prevailing image editing paradigm into the domain of video editing. This approach significantly reduces training complexity while preserving strong reference fidelity. Moreover, this conditioning strategy can be easily applied to other video backbone models, enabling flexible integration of video control signals without structural modifications. Furthermore, benefiting from the high input-output consistency inherent in subtitle removal tasks, we validate the feasibility of one-step sampling inference, which facilitates scalable deployment of the proposed framework.
  
  \item To enable the processing of video segments of arbitrary length, we adopt a chunk-wise strategy to handle the video. The chunk size is dynamically adjusted based on the conditional video's resolution, with higher resolutions corresponding to smaller chunk sizes. To mitigate temporal discontinuities between chunks, the last frame of the preceding chunk is used as the first frame of the noisy video latent in the subsequent chunk. Owing to the strong conditioning effect of the reference video, a single-frame condition is sufficient to achieve satisfactory temporal coherence. For tasks involving strong reference-based video editing, this is a simple yet effective strategy. 
\end{itemize}

\section{Related Work}

\subsection{Image Editing via Prompt Instruction}
In the field of image editing, InstructPix2Pix~\cite{instructpix2pix} has demonstrated the potential of the instruction-guided diffusion model for performing image editing tasks. 
Recently, \emph{FLUX.1 Kontext}~\cite{flux-kontext} has achieved significant success in the domain of instruction-guided image editing. It is a simple flow matching model that concatenates context and instruction tokens \emph{along the sequence dimension} and directly estimates a velocity prediction target. Very recently, \emph{Qwen-Image}~\cite{qwen-image} has also proposed a similar approach by incorporating the VAE-encoded latent representation of the input image into the image stream, concatenating it with the noisy image latent \emph{along the sequence dimension}. These suggest that, in the field of image editing, concatenating the conditional image latent with the noised image latent along the sequence dimension has become the prevailing paradigm. Inspired by these works, we investigate strategies for incorporating conditional video inputs into video editing.

\subsection{Diffusion-based Video Inpainting}
Diffusion-based video inpainting falls under the category of masked video-to-video editing (MV2V), which requires a mask sequence of the region of interest (ROI) $\bar{\mathbf{m}} \in \mathbb{R}^{f \times h \times w \times c_{m}}$. This approach completes masked video using a conditional diffusion model $u_{\theta } \left ( \mathbf{z}_t, \mathbf{z}_m, \bar{\mathbf{m}}, \text{prompt}, t  \right ) $, where $\mathbf{z}_t \in \mathbb{R}^{f \times h \times w \times c}$ is nosiy input,$\mathbf{z}_m \in \mathbb{R}^{f \times h \times w \times c}$ is masked video latent. Typically, $\mathbf{z}_t$, $\mathbf{z}_m$, and $\bar{m}$ are concatenated along the channel dimension as the denoising network's main stream inputs $\mathbf{z}_{in} \in \mathbb{R}^{f \times h \times w \times (c_m+2c)}$. Both Minimax-Remover~\cite{minimax-remover} and EraserDiT~\cite{eraserdit} adopt the paradigm of channel-wise concatenation. Minimax-Remover is built on the pretrained text-to-video generation model Wan2.1-1.3B~\cite{wan21}. To accommodate its design, the original input channel size of the backbone model is expanded from 16 to 48 ($c_m = 16$), which inevitably requires full-parameter fine-tuning. 

\section{Methodology}
\subsection{Preliminary}
\textbf{Flow matching.} Flow matching video generative models employ a neural network to synthesize realistic video frames by learning a time-dependent vector field that guides samples from a noise distribution toward the target video distribution. Given an input video $\mathbf{x}\in \mathbb{R} ^ {f \times h \times w \times c}$, a pretrained variational autoencoder (VAE) encoder $\mathcal{E}$ encodes $\mathbf{x}$ to latent representation $\mathbf{z_0}=\mathcal{E} \left ( \mathbf{x} \right ) $. The rectified flow-matching loss
\begin{equation}
    \mathcal{L}_{\theta} = \mathbb{E} _{t\sim p\left ( t \right ),\mathbf{z}_{0},\epsilon,c_\text{text} }\left [\left \| v_{\theta }\left ( \mathbf{z}_t,t,\mathbf{z}_\text{ref},c_\text{text} \right ) - \mathbf{v}_t  \right \|_{2}^{2} \right ], 
\label{eq:flow-matching}
\end{equation}
where $\mathbf{z}_t$ is the linearly interpolated latent between $\mathbf{z_0}$ and noise $\epsilon \in \mathcal{N} \left ( 0, \mathbf{I} \right ) $, $\mathbf{z}_t = (1 - t)\mathbf{z_0} + t\epsilon$, $\mathbf{z}_{ref}$ is the reference video latent $\mathbf{z}_{ref} = \mathbf{z}_{0} + \Delta_\text{subtitle}$, $c_\text{text}$ is text prompt, and the velocity $\mathbf{v}_t = \frac{d\mathbf{z}_t}{dt}=\epsilon - \mathbf{z}_{0}$.

\subsection{Model Architecture}
\begin{figure*}[ht]
  \centering
  \begin{subfigure}[c]{0.48\textwidth}
    \centering
    \adjustbox{valign=c}{\includegraphics[width=\textwidth]{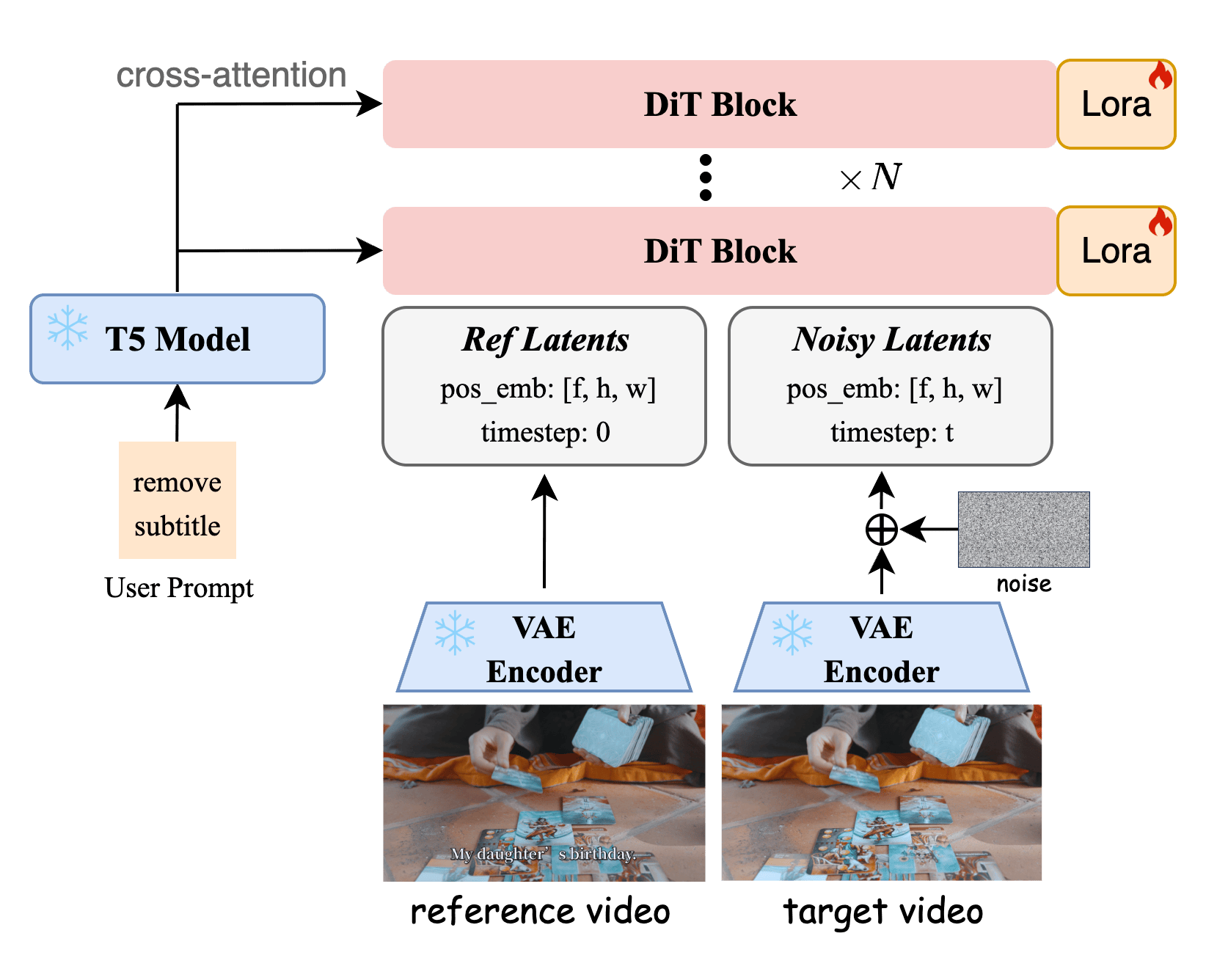}}
    \caption{The architecture of the training phase.}
    \label{fig:SEDiT_train_arch}
  \end{subfigure}
  \hfill
  \begin{subfigure}[c]{0.48\textwidth}
    \centering
    \adjustbox{valign=c}{\includegraphics[width=\textwidth]{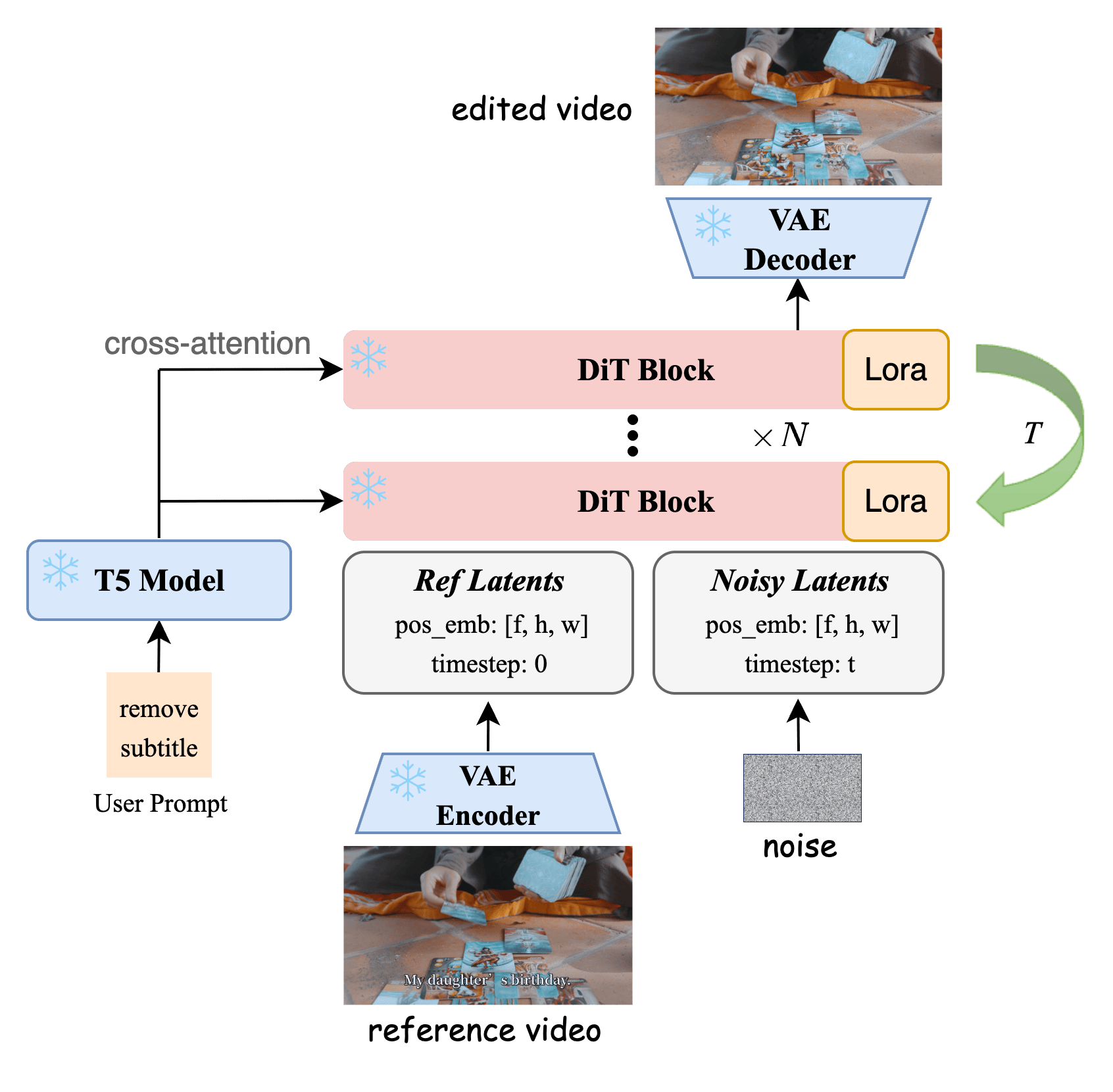}}
    \caption{The architecture of the inference phase.}
    \label{fig:SEDit_infer_arch}
  \end{subfigure}
  \caption{The overview of our mask-free video subtitle erasure framework. We introduce a conditional video branch alongside the original video branch. Given the conditional video, the VAE encoder maps it into tokens, which are concatenated with video latent tokens and then sent to the DiT.}
  \label{fig:SEDiT_framework}
\end{figure*}

The overall framework is illustrated in Figure~\ref{fig:SEDiT_framework}. We use LTX-Video-2B-0.9.6~\cite{ltx-video} as the video generation base model, which adopts a Diffusion Transformer (DiT) architecture. Our goal is to learn a model that can generate videos conditioned jointly on a text prompt and a reference video. Given a reference video with subtitles, we first encode it into the latent space using the pre-trained VAE encoder. The reference video latents are patchified into $\mathbf{z}_\text{ref} \in \mathbb{R}^{B \times (F \times H \times W) \times C}$, same as the noisy video tokens. Since the reference video latents remain clean through the denoising process, we set their denoising timestep to $0$ during diffusion. Then, the reference video tokens are concatenated with noisy video tokens along the sequence dimension and processed jointly through successive DiT blocks. The input video tokens can be expressed as $\mathbf{z}_{in} = \mathbb{R}^{B \times 2(FHW) \times C}$.

We encode positional information via 3D Rotary Positional Embedding (RoPE)~\cite{rope}, where the embeddings for the reference $\mathbf{z}_\text{ref}$ are aligned with those of the noisy tokens $\mathbf{z}_\text{noisy}$. Concretely, if a token position is denoted by the triplet $\mathbf{u} = (\mathrm{f}, \mathrm{h}, \mathrm{w})$, then we set $\mathbf{u}_\text{ref} = \mathbf{u}_\text{noisy} = (\mathrm{f}, \mathrm{h}, \mathrm{w})$ for both reference video tokens and target video tokens. The primary motivation for this design is to ensure frame-by-frame alignment between the reference video and the target video.

Since subtitles occupy only a small portion of the overall video content and both the reference and target videos share identical positional encodings, directly applying Equation~\ref{eq:flow-matching} to compute the loss tends to bias the model toward reconstructing the reference video. Therefore, we impose a higher loss penalty in the text regions to encourage the model to focus on subtitle removal. The focal loss function is defined by
\begin{equation}
    \mathcal{L}_\text{focal} = \mathcal{L}_{\theta } * \left ( \mathbf{I} + \alpha * \mathbf{M}_\text{subtitle} \right ) ,
\label{eq:loss}
\end{equation}
where $\mathbf{M}_\text{subtitle}$ indicates the mask obtained by filling the subtitle bounding box, $\mathbf{I}$ denotes an all-ones tensor with the same shape as $\mathbf{M}_\text{subtitle}$, and $\alpha$ is a scalar.

\subsection{Data Synthesis}\label{sec:data-synth}

\begin{figure*}[htpb]
\begin{center}
\includegraphics[width=0.95\textwidth]{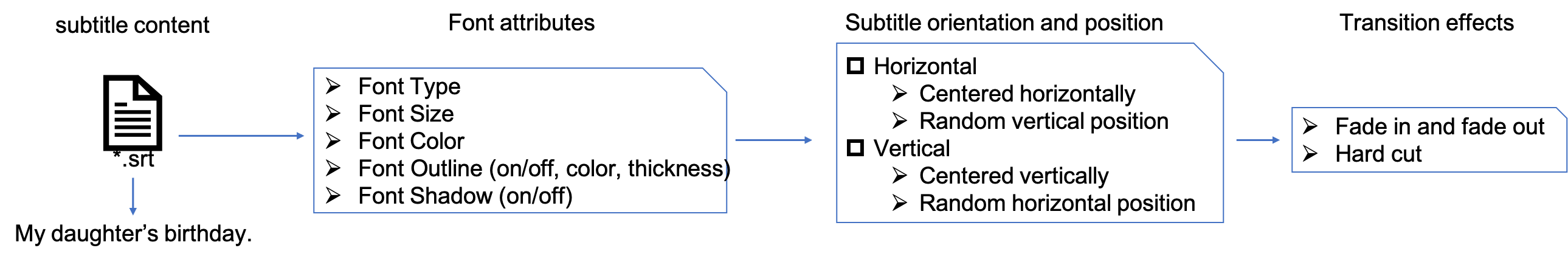}
\end{center}
\caption{Data Synthesis pipeline.}
\label{fig:data-synth}
\end{figure*}

To facilitate efficient model training, we simulate subtitles in the real-world videos through a rendering-based approach. To closely replicate the visual characteristics of subtitles in practical video scenarios, we construct a subtitle data synthesis pipeline, as illustrated in Figure~\ref{fig:data-synth}. The entire pipeline comprises four components: (1) subtitle content acquisition; (2) font attribute configuration; (3) subtitle layout specification; and (4) transition effect definition for subtitle switching. 

\textbf{Subtitle content acquisition.} We collected subtitle files (*.srt) from the internet spanning various genres, including films, anime, and television dramas. The dataset covers multiple languages, such as English, Chinese, Japanese, and Korean. We utilize the Python library \emph{pysrt} to extract and parse the detailed content of the subtitle files.

\textbf{font attribute configuration.} We collected a range of commonly used Chinese and English font files, including \emph{Source Han Sans}, \emph{Source Han Serif}, \emph{Alibaba PuHuiTi}, \emph{Arial}, and \emph{Helvetica}. We define the font size range as $[40pt, 150pt]$, and uniformly sample within this interval to obtain the current font size. Text colors are categorized into two distinct groups: bright-toned and dark-toned. Bright-toned text is defined by RGB values in which each channel—\textcolor{red}{red}, \textcolor{green}{green}, and \textcolor{blue}{blue}—is independently sampled from the range $[231, 255]$, resulting in high-luminance color compositions. In contrast, dark-toned text comprises RGB values uniformly sampled from the range $[0, 230]$ across all three channels, yielding lower-luminance appearances. For bright-toned text, the font outline color is set to a dark hue, with RGB values sampled from the range $[0, 15]$. Conversely, for dark-toned text, the outline color is assigned a light hue, with RGB values drawn from the interval $[231, 255]$. The thickness of the text border is determined probabilistically: with a probability of $50\%$, it is set to zero, indicating no visible outline. For the remaining $50\%$, the border width is uniformly sampled from the interval $[2pt, 5pt]$, producing a visible outline with variable thickness. In certain cases, subtitle text is rendered with shadow effects. To simulate this visual feature, we apply a slight positional offset to the font and combine it with a light black color, thereby producing a shadow-like appearance beneath the primary text.

\textbf{Subtitle layout specification.} Subtitles are typically arranged horizontally in the lower region of the video frame, regardless of whether the video is in landscape or portrait orientation. However, vertically oriented subtitles do occur occasionally. Accordingly, we assign a probability of $75\%$ to horizontal layout and $25\%$ to vertical layout. When subtitles are arranged horizontally, they are centered along the horizontal axis while their vertical position is randomly assigned. Conversely, when subtitles are arranged vertically, they are centered along the vertical axis with a randomly determined horizontal placement.

\textbf{Transition effect definition for subtitle switching.} Subtitle transitions in video content typically occur via two primary mechanisms: (1) abrupt switching and (2) fade-in/fade-out effects. To simulate the latter, we modulate the text opacity over time, thereby achieving a smooth transitional appearance. Additionally, in music video (MV) scenarios, lyric text often exhibits dynamic color changes synchronized with temporal progression. We replicate this behavior to accommodate lyric erasure effects commonly observed in MV-style presentations.

\subsection{Long Video Inference}

\begin{figure}[htbp]
  \centering
  \begin{minipage}[t]{0.45\textwidth}
    \centering
    \includegraphics[width=\textwidth]{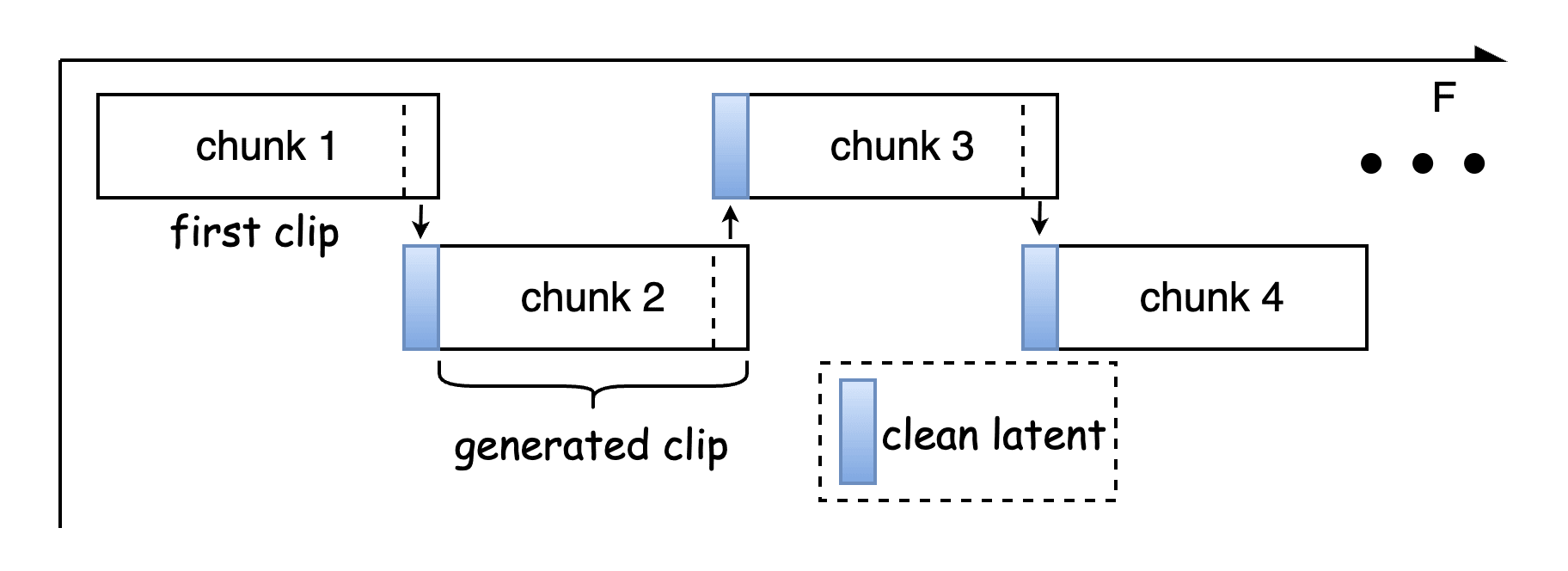}
    \caption{Long video inference strategy.}
    \label{fig:long_infer_startegy}
  \end{minipage}
  \hfill
  \begin{minipage}[t]{0.3\textwidth}
    \centering
    \includegraphics[width=\textwidth]{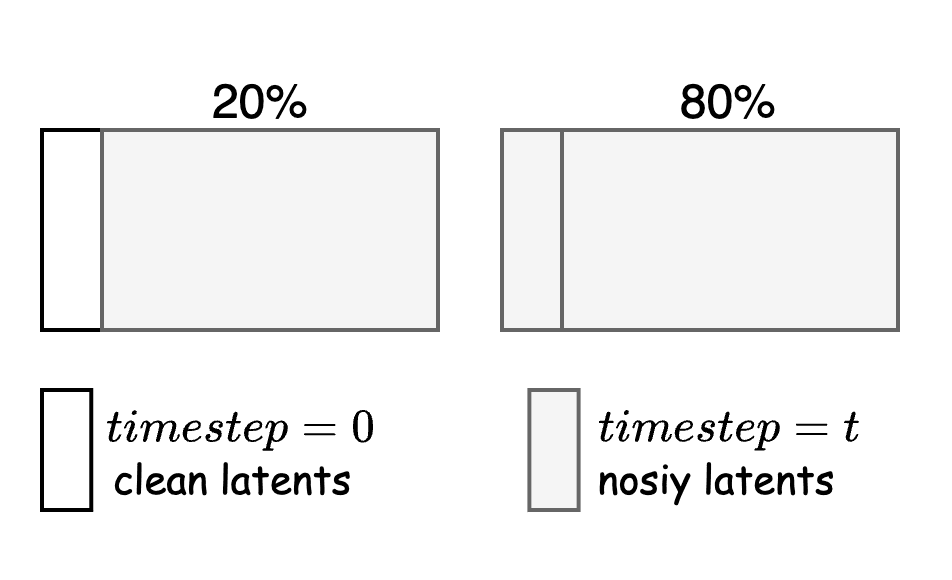}
    \caption{First-frame conditioning.}
    \label{fig:first-frame-cond}
  \end{minipage}
\end{figure}

In practical subtitle removal tasks, it is common to process videos exceeding $10$ minutes in duration, such as television dramas and animated content. We begin by applying scene detection algorithms TransNetV2~\cite{transnetv2} to segment the entire video into individual shots, followed by shot-wise processing. For short shots ($\le 5$ seconds), the model can handle them directly. However, for longer shots ($>5$ seconds), we adopt a chunk-wise processing strategy. The chunk size is determined based on the video resolution: specifically, $121$ frames for 720p videos, $65$ frames for 1080p videos, and $41$ frames for 1440p videos. To support any-length video processing, we set the first-frame conditioning, as shown in Figure~\ref{fig:first-frame-cond}. During the training phase, the noised latent representation of the first frame is replaced with its clean counterpart with a probability of $0.2$, and the corresponding diffusion timestep is set to zero. During inference, except for the first chunk, the first frame of each subsequent chunk is initialized using the last frame of the preceding chunk. Experimental results demonstrate that overlapping by a single frame is sufficient to maintain satisfactory temporal consistency. To reduce memory consumption, frame data is written into a TS (Transport Stream) video file in a streaming manner upon completion of each shot.

To reduce the burden of complex prompt formulation during inference, we employ a predefined prompt during training that omits the specific content of the subtitles. The prompt is defined as: \emph{``Please remove the subtitle text from the video while preserving the character appearance, background composition, and color style. Do not add any new elements.''} Since the current task is well-defined and the prompt remains unchanged, the prompt merely serves as a necessary input for the base model and has almost no impact on the video subtitle erasure task.

By default, our method performs automatic subtitle removal without requiring user interaction. However, in certain scenarios, users may wish to erase subtitles only within a specific region while preserving the rest of the content. Our model naturally supports this form of region-constrained editing. Given a user-specified binary mask $\mathbf{M}_\text{user}$, we perform region protection directly in the VAE latent space. The final output latent is computed as $\mathbf{z}_\text{out} = \text{SEDiT}\left ( \mathbf{z}_\text{ref}, \mathbf{z}_\text{in} \right ) \cdot \mathbf{M}_\text{latent} + \mathbf{z}_\text{ref} \cdot \left ( 1 - \mathbf{M}_\text{latent} \right ) $, where $\mathbf{M}_\text{latent} = \text{Interpolate}(\mathbf{M}_\text{user}, (F_{latent}, H_{latent}, W_{latent}))$ denotes the spatial-temporal interpolated mask aligned with the latent resolution. This formulation ensures that only the user-specified region is modified while the remaining content is faithfully preserved. As illustrated in Figure~\ref{fig:logo_protect}, we apply a region protection mechanism to the logo area, which allows greater flexibility in meeting user requirements. At present, \textbf{SEDiT} is well compatible with previous mask-based models.

\begin{figure}[!ht]
    \centering
    \includegraphics[width=0.48\textwidth]{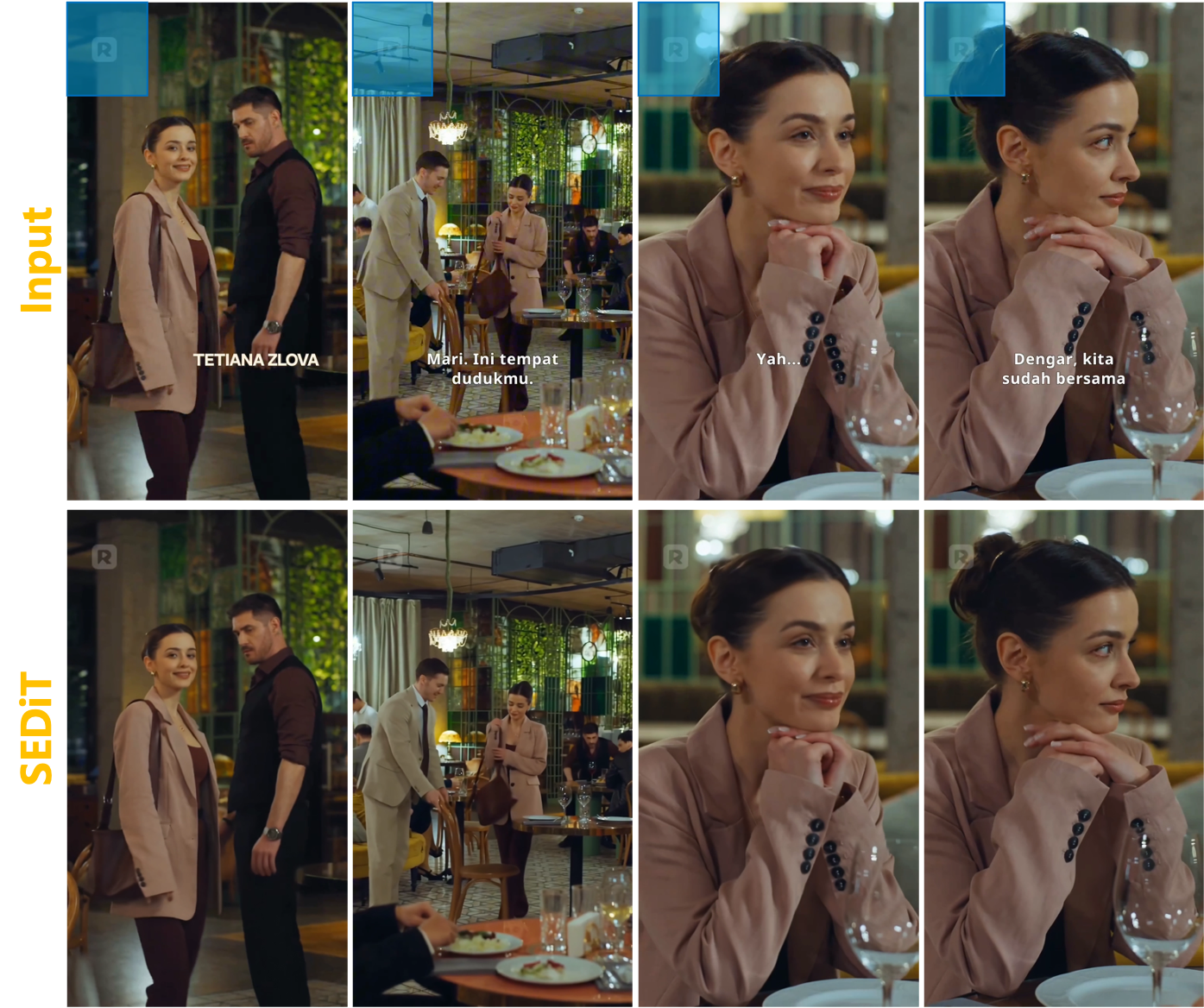}
    \caption{Example of Logo Protection. The blue box indicates the protected region. \emph{Best viewed when zoomed in.}}
    \label{fig:logo_protect}
\end{figure}

\subsection{Theoretical Justification for One-Step Sampling}
\label{sec:theory}

The empirical success of one-step inference in SEDiT is grounded in a 
theoretical analysis of the conditional rectified flow (CRF) framework 
applied to the subtitle erasure task.

\paragraph{Localized distributional shift.}
In subtitle removal, the reference video $z_\text{ref}$ and the target 
subtitle-free video differ only within the subtitle mask region $\mathcal{M}$. 
Formally, denoting the reference distribution as $q(z \mid z_\text{ref})$ 
and the target distribution as $p(z \mid z_\text{ref})$, the two 
distributions are identical outside $\mathcal{M}$. The conditional optimal 
transport (OT) map therefore decomposes as:
\begin{equation}
    T^*(z \mid z_\text{ref}) 
    = \bigl(T^*_{\mathcal{M}}(z_{\mathcal{M}},\, z_{\neg\mathcal{M}} \mid z_\text{ref}),\ 
             z_{\neg\mathcal{M}}\bigr),
    \label{eq:ot_decomp}
\end{equation}
where the non-subtitle region $z_{\neg\mathcal{M}}$ is preserved exactly by the OT map, reflecting the constraint that background content outside the subtitle region should remain unaltered.

\paragraph{Lipschitz continuity.}
Under a mild local Lipschitz assumption on $T^*_{\mathcal{M}}$ (see the Supplementary File for the complete proof), the conditional OT map satisfies:
\begin{equation}
    \bigl\|T^*(z^{(1)} \mid z_\text{ref}) - T^*(z^{(2)} \mid z_\text{ref})\bigr\| 
    \leq \sqrt{L_{\mathcal{M}}^2 + 1}\;\bigl\|z^{(1)} - z^{(2)}\bigr\|,
    \label{eq:lipschitz_ot}
\end{equation}
and the induced velocity field $v^*(z) = T^*(z \mid z_\text{ref}) - z$ 
is also Lipschitz with constant $L_T + 1$, where 
$L_T = \sqrt{L_{\mathcal{M}}^2 + 1}$.
Since $\mathcal{M}$ occupies only a small spatio-temporal region, 
$L_{\mathcal{M}}$ remains moderate, and the CRF path is globally 
near-linear with only localized curvature within $\mathcal{M}$.

\paragraph{Implication for one-step inference.}
The near-linearity of the conditional OT path means that a single 
Euler integration step:
\begin{equation}
    z = z_\text{noisy} + v_\theta(z_\text{noisy},\, t \mid z_\text{ref}) \cdot \Delta t
    \label{eq:one_step}
\end{equation}
suffices to accurately approximate the full flow from the noisy latent to the subtitle-free latent.
This provides a theoretical guarantee that one-step denoising is not merely an engineering shortcut, but a principled consequence of the localized and smooth structure inherent in the subtitle removal task. The complete formal proof is deferred to the Supplementary File.

\section{Experiments}\label{sec:experiments}
\subsection{Datasets}
We collected $400$K high-definition, subtitle-free videos from the Pexels website\footnote{\url{https://www.pexels.com/}} to serve as ground truth (GT) data. Embedded subtitle videos were subsequently generated on-the-fly using the data synthesis pipeline described in Section~\ref{sec:data-synth}.
Due to the absence of publicly available benchmark datasets for video subtitle removal, we construct a dataset of 400 samples to facilitate fair comparison across models. Each sample includes a clean video without subtitles, a corresponding subtitle mask, and a version with embedded subtitles. In the dataset, videos with 720p resolution contain $121$ frames, while those with 1080p resolution consist of $81$ frames. We name this dataset the Video Subtitle Removal Benchmark (VSR-Bench-400) dataset.

\subsection{Implementation details}
We adopt Low-Rank Adaptation (LoRA) with rank $256$. For the 2B-parameter LTX-Video-0.9.6 model, this adds just 381M trainable parameters ($19\%$ of the base model). Training uses a batch size of $32$. For input resolutions close to $1080 \times 1920$, the input frame length is set to $65$; for resolutions near $1280 \times 720$, the input length is set to $121$. The training process comprises $100K$ iterations on 8 NVIDIA A800 (80GB) GPUs, utilizing AdamW optimizer with an initial learning rate of $2 \times 10 ^ {-4}$. $\alpha$ is experimentally set to $5$ in Equation~\ref{eq:loss}. During inference, we experimentally use one sampling step without model distillation. This is primarily attributed to the strong conditional video input provided to the model, whereby regions outside the subtitle area are largely reconstructed through a \emph{copy-paste mechanism}.

\begin{figure}[!t]
    \centering
    \includegraphics[width=0.5\textwidth]{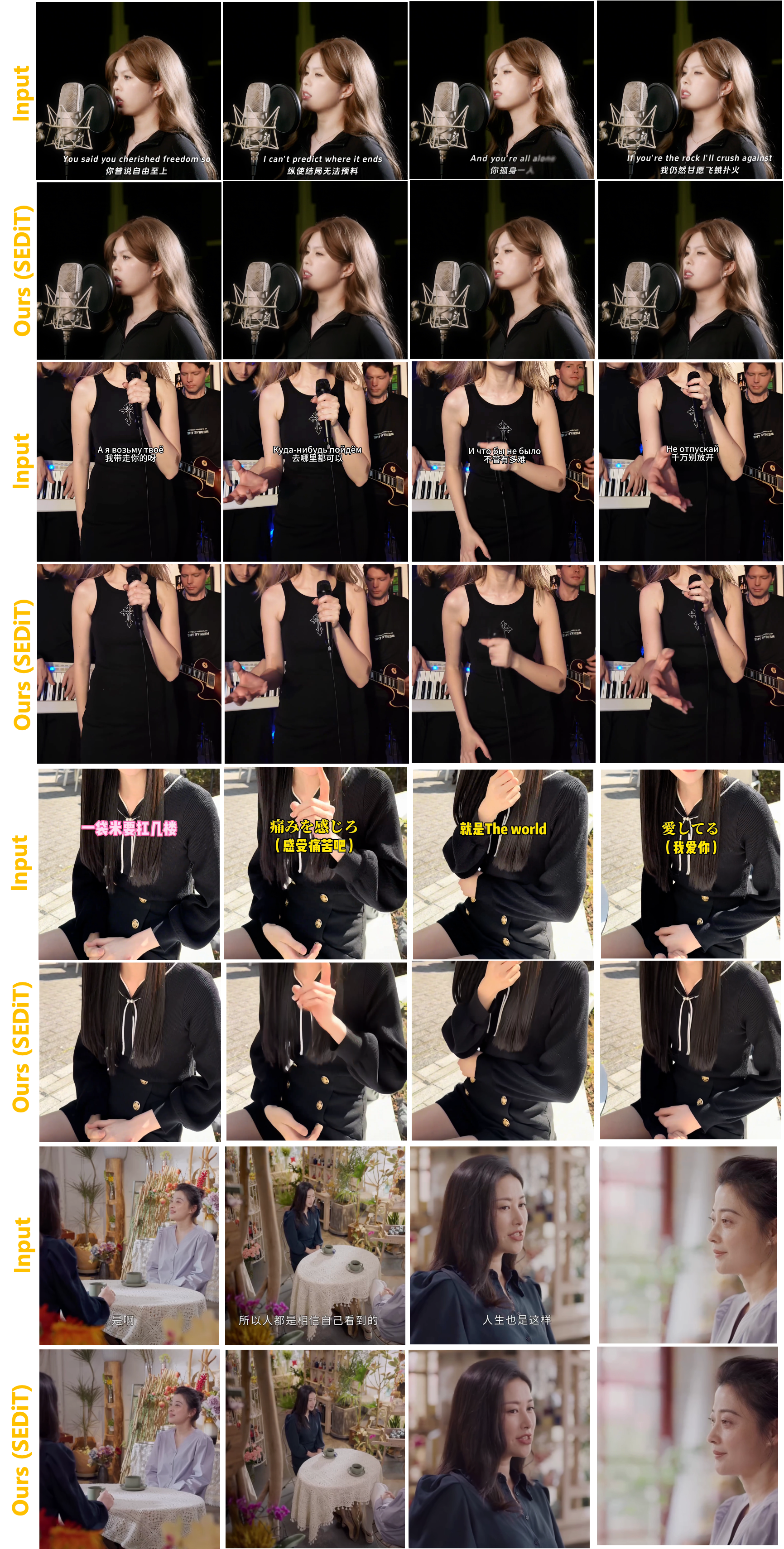}
    \caption{The evaluation results on multilingual subtitles include cases with mixed Chinese-English subtitles, mixed Chinese-Russian subtitles, Japanese subtitles, and Chinese-only subtitles. \emph{Best viewed when zoomed in.}}
    \label{fig:multi-language_demo}
\end{figure}


\begin{table*}[htbp]
\centering
\caption{Qualitative results on the VSR-Bench-400 dataset. ``Time'' indicates the average inference time per video with a resolution of $1920 \times 1080$ with $65$ frames \textbf{without any acceleration method} on A800. The inference time only accounts for the algorithm itself, excluding the time for video reading and writing. ``N/A'' denotes not available. Mask-based approaches adopt \textbf{GT masks}. The evaluation is conducted directly on the raw outputs of the model, without any post-processing operations added.}
\label{tab:metrics}
\begin{tabular}{lcccccc}
\Xhline{1.2pt}
\multicolumn{1}{c}{} & \multicolumn{5}{c}{\textbf{VSR-Bench-400}} \\
\makecell{\textbf{Method}} & \textbf{PSNR} ↑ & \textbf{SSIM} ↑ & \textbf{LPIPS} ↓ & \textbf{FVD} ↓ & \textbf{MOS} ↑ & \textbf{Time} ↓\\
\Xhline{0.8pt}
\hline
\makecell{Propainter} & 26.8198 & 0.8231 & 0.1778 & 64.9960 & N/A & 38s \\
\makecell{DiffuEraser} & 27.5084 & 0.8146 & 0.1223 & 56.7705 & N/A & 166s \\
\makecell{Minimax-Remover (6-step) }     & 28.3109           & 0.8785          & 0.1011     &    39.3450   & 2.5      &   150s \\
\hline
\makecell{SEDiT (4-step)} & 31.2961 & 0.8783 & 0.1001 & 26.7560 & N/A & 8s \\
\makecell{SEDiT (2-step)} & \cellcolor{pink}31.5968 & \cellcolor{pink}0.8805 & 0.0982 & 24.9137 & N/A & 6s \\
\makecell{SEDiT (1-step)}         & 31.5863           & 0.8805           & \cellcolor{pink}0.0981      &   \cellcolor{pink}24.0599  & \cellcolor{pink}4.5     &   \cellcolor{pink}4s \\
\Xhline{1.2pt}
\end{tabular}
\end{table*}

\subsection{Quantitative experiments}
We compare our method against Minimax-Remover~\cite{minimax-remover}, a leading open-source video object removal approach based on the DiT architecture. Besides, we compare the representative GAN-based method Propainter~\cite{propainter} and the diffusion-UNet-based method Diffueraser~\cite{diffueraser}. 
For the mask-based methods, we use subtitle boundary filling to generate rectangular masks (GT masks). The masked video input to the model ($V_\text{masked} = V \cdot (\mathbf{1} - \mathbf{M}) + \mathbf{0} \cdot \mathbf{M}$) does not contain visible subtitles and is independent of the subtitle rendering process. In addition, the VSR-Bench-400 dataset employs original videos different from those used in training, specifically to avoid unfair comparisons caused by potential data leakage.
For visual quality assessment, we employ SSIM (Structural Similarity Index Measure), LPIPS(Learned Perceptual Image Patch Similarity) to measure frame-level fidelity. To assess the perceptual quality of subtitle removal results, we conducted a Mean Opinion Score (MOS) evaluation involving $20$ human participants. As presented in Table~\ref{tab:metrics}, our proposed method consistently outperforms Minimax-Remover across all considered evaluation metrics. Notably, our method exhibits significantly superior performance in both inference time and average MOS. We experimented with sampling steps of 1, 2, and 4. Based on objective metrics, both $step=2$ and $step=1$ have their respective advantages. Considering efficiency, we ultimately adopted the $step=1$ configuration.

\subsection{Qualitative experiments}

As shown in Figure~\ref{fig:multi-language_demo}, the proposed method demonstrates effective and seamless subtitle removal. In the third column of the top-left example, despite the presence of pronounced subtitle-specific visual effects, our approach successfully eliminates the text without leaving visible traces. Such stylized and blurred subtitles are typically challenging for conventional OCR algorithms, which often suffer from low recognition accuracy and missed detections.

The bottom-left example highlights the robustness of our model: even though Russian subtitles were not included in the training data, the method still achieves high-quality removal. The top-right case further illustrates the generalization capability of our approach when handling colorful, bordered subtitles that occupy a large portion of the frame and obscure facial regions.

In the bottom-right example, the model excels at restoring complex textures, such as patterned tablecloths. For frames without subtitles, our method accurately identifies the absence of text and avoids unnecessary alterations. Overall, the results yield visually coherent reconstructions with minimal artifacts.

\begin{figure}[ht]
    \centering
    \includegraphics[width=0.5\textwidth]{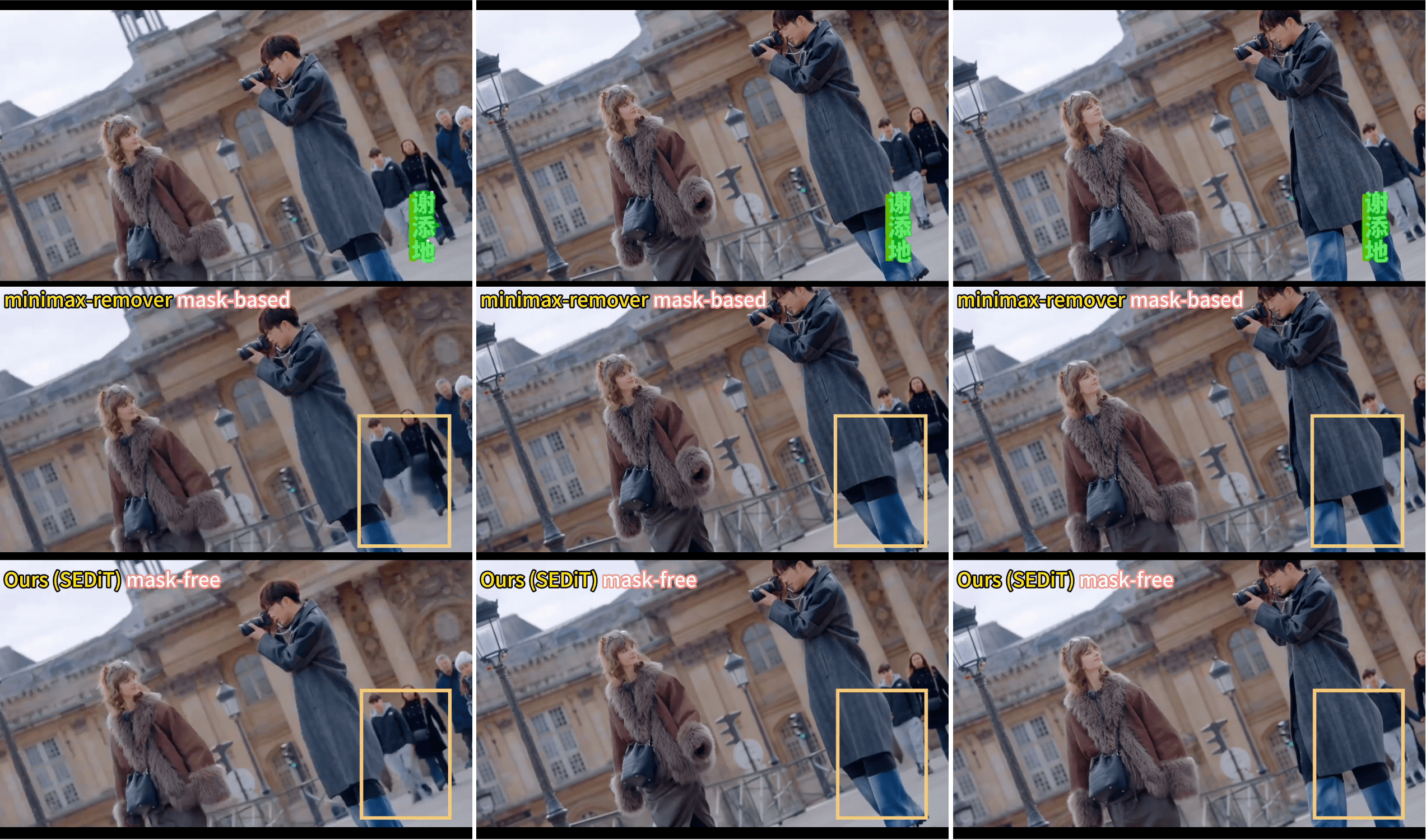}
    \caption{Visual comparison with state-of-the-art (SOTA) methods. The green highlighted region represents the subtitle mask, which is obtained using the SAM2 model. Focus on the visual comparison within the yellow boxes. \emph{Best viewed when zoomed-in.}}
    \label{fig:compare_demo1}
\end{figure}

\begin{figure}[ht]
    \centering
    \includegraphics[width=0.5\textwidth]{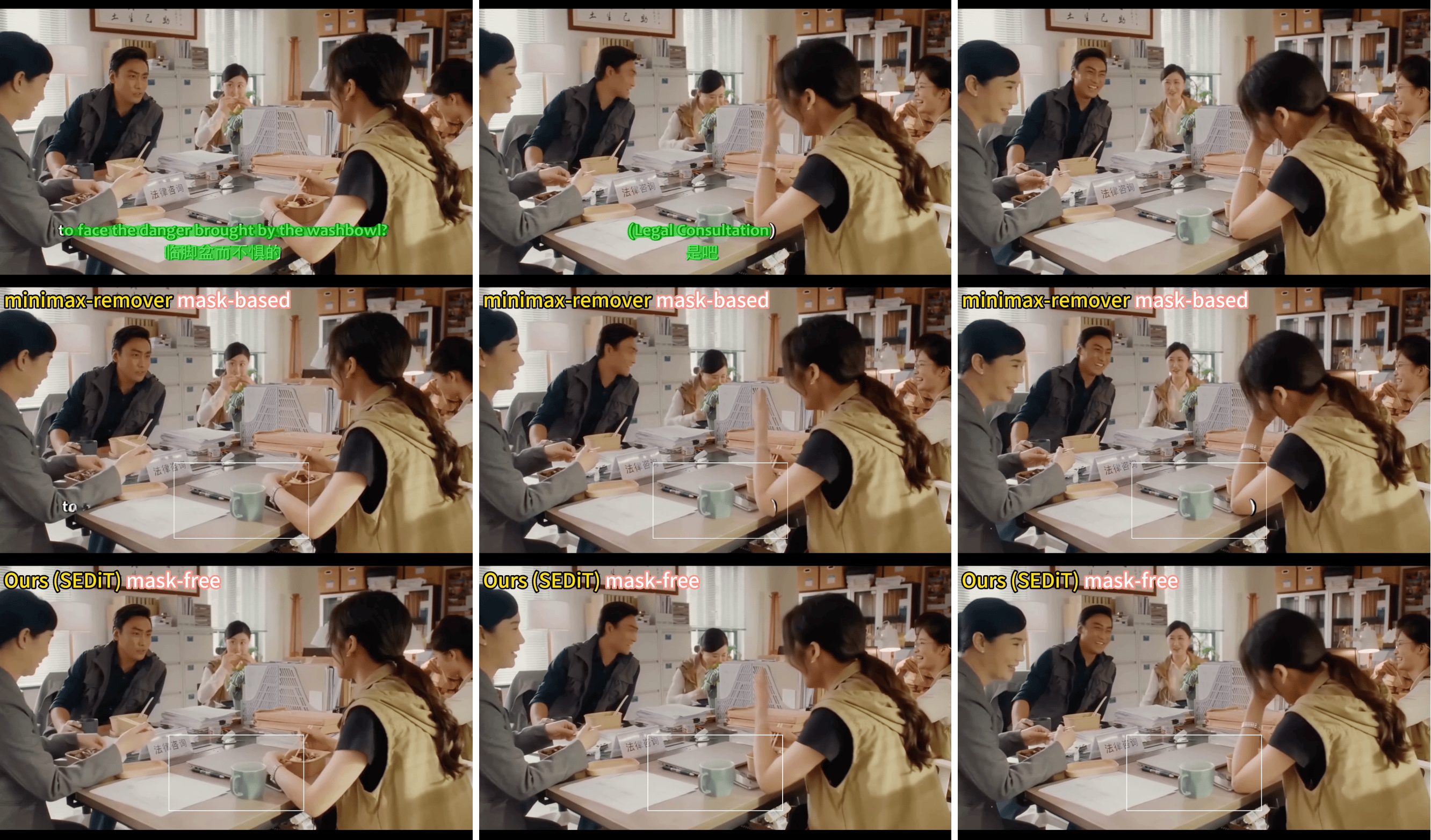}
    \caption{Visual comparison with state-of-the-art (SOTA) methods. The green highlighted region represents the subtitle mask, which is obtained using the SAM2 model. Focus on the visual comparison within the white boxes. \emph{Best viewed when zoomed-in.}}
    \label{fig:compare_demo2}
\end{figure}

To compare our approach with state-of-the-art mask-based video inpainting methods under real-world conditions, we utilize an interactive interface to obtain subtitle locations and apply SAM2~\cite{sam2} for subtitle tracking. The resulting masks are then fed into the Minimax-Remover~\cite{minimax-remover} method to generate subtitle removal results. Compared to traditional OCR-based subtitle box extraction, this segmentation-based strategy yields more precise subtitle regions, which benefits mask-guided inpainting methods.

As illustrated in Figure~\ref{fig:compare_demo1}, our mask-free approach successfully reconstructs occluded body parts of distant characters, whereas the baseline method introduces noticeable blurring artifacts. The final column of Figure~\ref{fig:compare_demo1} further shows that, in cases of mild occlusion, the baseline method can still produce reasonably good reconstructions.

As shown in Figure~\ref{fig:compare_demo2}, when the subtitle mask is inaccurately defined, mask-based methods are prone to error propagation. Comparing the second and third columns, the incomplete mask in the second column results in residual subtitle artifacts in the generated output. This error propagates further to the third column, even though the input frame in that case contains no subtitles. \textbf{Additional high-resolution visual results are provided in Figure~\ref{fig:vsr_bench_1}, Figure~\ref{fig:vsr_bench_2} and Figure~\ref{fig:vsr_bench_3}. }

\subsection{Ablation study}
\begin{figure}[t]
    \centering
    \includegraphics[width=0.5\textwidth]{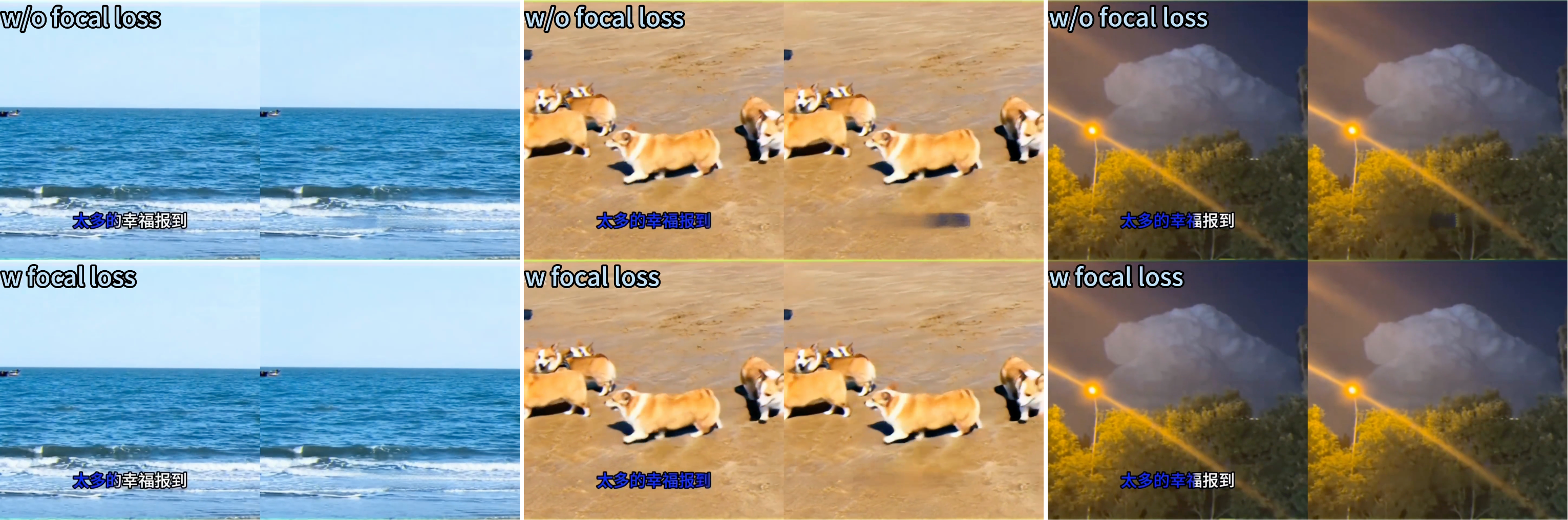}
    \caption{Ablation study for the focal loss. ``w/o focal loss'' denotes ``without focal loss''. ``w focal loss'' indicates ``with focal loss''.}
    \label{fig:ablation_study}
\end{figure}

\begin{table}[h]
\centering
\caption{Quantitative comparison of conditioning methods}
\begin{tabular}{lcc}
\toprule
\textbf{Metric} & \textbf{In-context} & \textbf{Channel-conatenation} \\
\midrule
Training steps & 3,000 & 3,000 \\
LoRA rank & 256 & 256 \\
Inference steps & 1 & 1 \\
\midrule
PSNR $\uparrow$ & \cellcolor{pink}31.2999 & 27.8222 \\
SSIM $\uparrow$ & \cellcolor{pink}0.8784 & 0.8392 \\
LPIPS $\downarrow$ & \cellcolor{pink}0.0998 & 0.1247 \\
FVD $\downarrow$ & \cellcolor{pink}26.4344 & 51.8858 \\
\bottomrule
\end{tabular}
\label{tab:conditioning_comparison}
\end{table}

\textbf{Effectiveness of focal loss.} We experiment by setting $\alpha = 0$ in Equation~\ref{eq:loss}, which indicates that we directly use the original conditional rectified flow-matching loss to train our model. 
The experiment presented in Figure~\ref{fig:ablation_study} indicates that this approach does not completely remove subtitles when dealing with complex styles such as gradient-colored text. By incorporating focal loss, the model is encouraged to concentrate more effectively on subtitle regions, resulting in significantly improved removal performance.

\begin{figure}[t]
    \centering
    \includegraphics[width=0.48\textwidth]{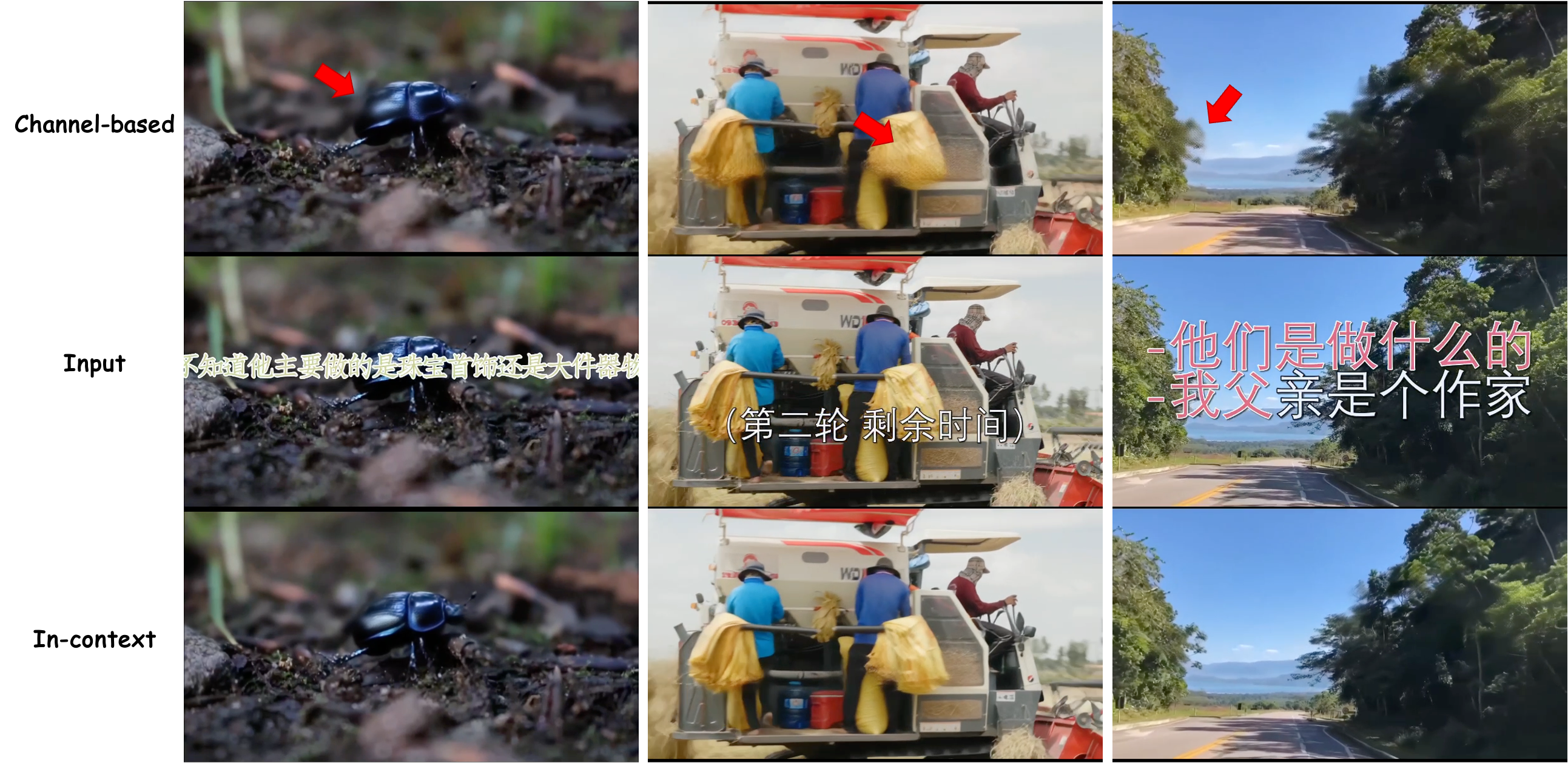}
    \caption{Channel-concatenation conditioning tends to introduce noticeable blurring or artifacts, whereas in‑context conditioning can reconstruct the background content occluded by subtitles more faithfully.}
    \label{fig:condition_style}
\end{figure}

\textbf{Comparison of Condition Injection Methods.} As illustrated in Table~\ref{tab:conditioning_comparison}, we compare two approaches for injecting conditions: in-context and channel-concatenation. Across all four metrics (PSNR, SSIM, LPIPS, and FVD), in‑context conditioning consistently outperforms channel‑concatenation conditioning. Since both methods use the same 1‑step sampling strategy, the performance gap cannot be attributed to sampling differences, providing strong empirical support for our theoretical claim that 1‑step diffusion is well‑suited for localized video editing tasks. We also provide a visual comparison between the two conditioning strategies (see Fig~\ref{fig:condition_style}). The results show that the channel‑concatenation approach often introduces noticeable blurring or artifacts around the edited regions, whereas in‑context conditioning more faithfully reconstructs the background content originally occluded by subtitles.

\section{Limitations}

While SEDiT demonstrates strong performance in removing subtitles across most video scenarios, there are still some limitations. Specifically, we found that SEDiT may fail to fully remove static subtitles in extremely short shots (less than $10$ frames). We observe that the current model is unable to remove text with severe motion blur, as illustrated in Figure~\ref{fig:failure_case}(the first two rows). Nevertheless, for typical clear subtitles, satisfactory removal results can be obtained even with relatively large font sizes.

\begin{figure}[t]
    \centering
    \includegraphics[width=0.48\textwidth]{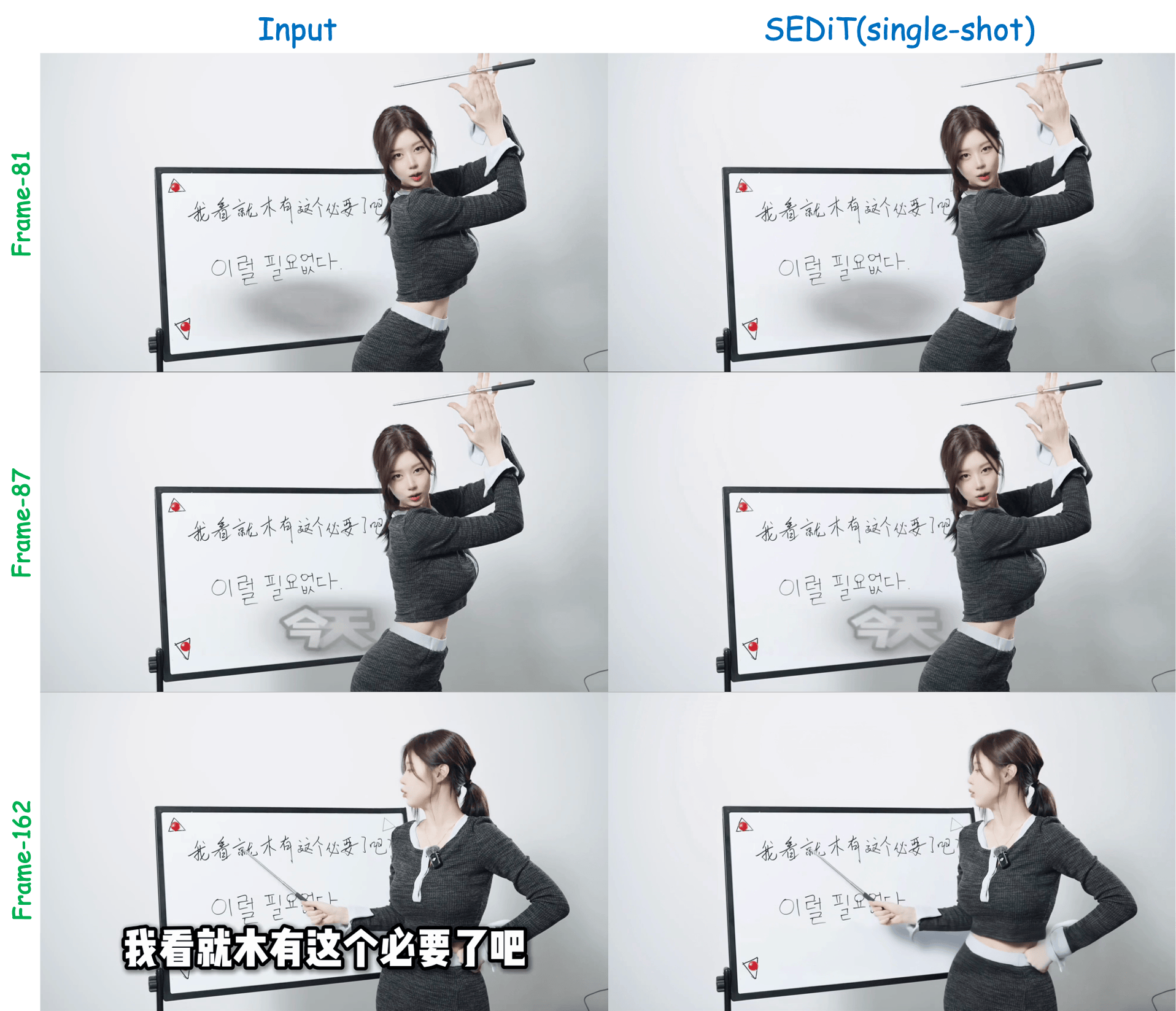}
    \caption{Visual results of processing severely motion-blurred text.}
    \label{fig:failure_case}
\end{figure}

While SEDiT demonstrates strong performance in removing subtitles across most video scenarios, there are still some limitations. Specifically, we found that SEDiT may fail to fully remove static subtitles in extremely short shots (less than $10$ frames). We observe that the current model is unable to remove text with severe motion blur, as illustrated in Figure~\ref{fig:failure_case}(the first two rows). Nevertheless, for typical clear subtitles, satisfactory removal results can be obtained even with relatively large font sizes.

\section{Conclusion}

We propose \textbf{SEDiT}, a lightweight, one-stage, mask-free video editing framework for high-field video subtitle erasure. Unlike previous mask-based video inpainting methods, our approach eliminates the need for explicit masks, thereby bypassing the video mask segmentation process entirely. Furthermore, our method preserves the architecture of the underlying generative model, enabling efficient fine-tuning via LoRA. This design choice also allows for seamless integration with higher-quality video generation backbones to further enhance performance. 

To support subtitle removal in videos of arbitrary length, we adopt a chunk-wise processing strategy, dynamically adjusting the chunk size based on the input resolution. To mitigate temporal discontinuities across chunks, we probabilistically inject the initial frame as a conditioning signal. Benefiting from this strong conditional guidance, frame repetition across chunks achieves temporal consistency with minimal artifacts. 

During inference, our method requires only one step to produce high-quality results. With the high compression ratio of the base VAE model, our approach can directly process 1080p-resolution videos, completing 65 frames in just $4$ seconds—making it particularly suitable for large-scale deployment.

\bibliography{example_paper}

@article{stable-video-diffusion,
  title={Stable video diffusion: Scaling latent video diffusion models to large datasets},
  author={Blattmann, Andreas and Dockhorn, Tim and Kulal, Sumith and Mendelevitch, Daniel and Kilian, Maciej and Lorenz, Dominik and Levi, Yam and English, Zion and Voleti, Vikram and Letts, Adam and others},
  journal={arXiv preprint arXiv:2311.15127},
  year={2023}
}

@inproceedings{dit,
  title={Scalable diffusion models with transformers},
  author={Peebles, William and Xie, Saining},
  booktitle={Proceedings of the IEEE/CVF international conference on computer vision},
  pages={4195--4205},
  year={2023}
}

@article{wan21,
  title={Wan: Open and advanced large-scale video generative models},
  author={Wan, Team and Wang, Ang and Ai, Baole and Wen, Bin and Mao, Chaojie and Xie, Chen-Wei and Chen, Di and Yu, Feiwu and Zhao, Haiming and Yang, Jianxiao and others},
  journal={arXiv preprint arXiv:2503.20314},
  year={2025}
}

@article{hunyuanvideo,
  title={Hunyuanvideo: A systematic framework for large video generative models},
  author={Kong, Weijie and Tian, Qi and Zhang, Zijian and Min, Rox and Dai, Zuozhuo and Zhou, Jin and Xiong, Jiangfeng and Li, Xin and Wu, Bo and Zhang, Jianwei and others},
  journal={arXiv preprint arXiv:2412.03603},
  year={2024}
}

@article{ltx-video,
  title={Ltx-video: Realtime video latent diffusion},
  author={HaCohen, Yoav and Chiprut, Nisan and Brazowski, Benny and Shalem, Daniel and Moshe, Dudu and Richardson, Eitan and Levin, Eran and Shiran, Guy and Zabari, Nir and Gordon, Ori and others},
  journal={arXiv preprint arXiv:2501.00103},
  year={2024}
}

@article{eraserdit,
  title={EraserDiT: Fast Video Inpainting with Diffusion Transformer Model},
  author={Liu, Jie and Hui, Zheng},
  journal={arXiv preprint arXiv:2506.12853},
  year={2025}
}

@article{minimax-remover,
  title={MiniMax-Remover: Taming Bad Noise Helps Video Object Removal},
  author={Zi, Bojia and Peng, Weixuan and Qi, Xianbiao and Wang, Jianan and Zhao, Shihao and Xiao, Rong and Wong, Kam-Fai},
  journal={arXiv preprint arXiv:2505.24873},
  year={2025}
}

@article{diffueraser,
  title={Diffueraser: A diffusion model for video inpainting},
  author={Li, Xiaowen and Xue, Haolan and Ren, Peiran and Bo, Liefeng},
  journal={arXiv preprint arXiv:2501.10018},
  year={2025}
}

@inproceedings{propainter,
  title={Propainter: Improving propagation and transformer for video inpainting},
  author={Zhou, Shangchen and Li, Chongyi and Chan, Kelvin CK and Loy, Chen Change},
  booktitle={Proceedings of the IEEE/CVF international conference on computer vision},
  pages={10477--10486},
  year={2023}
}

@article{qwen2.5-vl,
  title={Qwen2.5-vl technical report},
  author={Bai, Shuai and Chen, Keqin and Liu, Xuejing and Wang, Jialin and Ge, Wenbin and Song, Sibo and Dang, Kai and Wang, Peng and Wang, Shijie and Tang, Jun and others},
  journal={arXiv preprint arXiv:2502.13923},
  year={2025}
}

@article{sam2,
  title={Sam 2: Segment anything in images and videos},
  author={Ravi, Nikhila and Gabeur, Valentin and Hu, Yuan-Ting and Hu, Ronghang and Ryali, Chaitanya and Ma, Tengyu and Khedr, Haitham and R{\"a}dle, Roman and Rolland, Chloe and Gustafson, Laura and others},
  journal={arXiv preprint arXiv:2408.00714},
  year={2024}
}

@article{flux-kontext,
  title={FLUX. 1 Kontext: Flow Matching for In-Context Image Generation and Editing in Latent Space},
  author={Labs, Black Forest and Batifol, Stephen and Blattmann, Andreas and Boesel, Frederic and Consul, Saksham and Diagne, Cyril and Dockhorn, Tim and English, Jack and English, Zion and Esser, Patrick and others},
  journal={arXiv preprint arXiv:2506.15742},
  year={2025}
}

@article{qwen-image,
  title={Qwen-image technical report},
  author={Wu, Chenfei and Li, Jiahao and Zhou, Jingren and Lin, Junyang and Gao, Kaiyuan and Yan, Kun and Yin, Sheng-ming and Bai, Shuai and Xu, Xiao and Chen, Yilei and others},
  journal={arXiv preprint arXiv:2508.02324},
  year={2025}
}

@inproceedings{instructpix2pix,
  title={Instructpix2pix: Learning to follow image editing instructions},
  author={Brooks, Tim and Holynski, Aleksander and Efros, Alexei A},
  booktitle={Proceedings of the IEEE/CVF conference on computer vision and pattern recognition},
  pages={18392--18402},
  year={2023}
}

@inproceedings{transnetv2,
  title={Transnet v2: An effective deep network architecture for fast shot transition detection},
  author={Soucek, Tom{\'a}s and Lokoc, Jakub},
  booktitle={Proceedings of the 32nd ACM International Conference on Multimedia},
  pages={11218--11221},
  year={2024}
}

@article{rope,
  title={Roformer: Enhanced transformer with rotary position embedding},
  author={Su, Jianlin and Ahmed, Murtadha and Lu, Yu and Pan, Shengfeng and Bo, Wen and Liu, Yunfeng},
  journal={Neurocomputing},
  volume={568},
  pages={127063},
  year={2024},
  publisher={Elsevier}
}

@article{ditto-1m,
  title={Scaling Instruction-Based Video Editing with a High-Quality Synthetic Dataset},
  author={Bai, Qingyan and Wang, Qiuyu and Ouyang, Hao and Yu, Yue and Wang, Hanlin and Wang, Wen and Cheng, Ka Leong and Ma, Shuailei and Zeng, Yanhong and Liu, Zichen and others},
  journal={arXiv preprint arXiv:2510.15742},
  year={2025}
}

@article{paddleocr-vl,
  title={Paddleocr-vl: Boosting multilingual document parsing via a 0.9 b ultra-compact vision-language model},
  author={Cui, Cheng and Sun, Ting and Liang, Suyin and Gao, Tingquan and Zhang, Zelun and Liu, Jiaxuan and Wang, Xueqing and Zhou, Changda and Liu, Hongen and Lin, Manhui and others},
  journal={arXiv preprint arXiv:2510.14528},
  year={2025}
}

@article{CLEAR,
  title={CLEAR: Context-Aware Learning with End-to-End Mask-Free Inference for Adaptive Video Subtitle Removal},
  author={He, Qingdong and Wang, Chaoyi and Tang, Peng and Yang, Yifan and Hu, Xiaobin},
  journal={arXiv preprint arXiv:2603.21901},
  year={2026}
}

@article{kiwi-edit,
  title={Kiwi-Edit: Versatile Video Editing via Instruction and Reference Guidance},
  author={Lin, Yiqi and Liang, Guoqiang and Zeng, Ziyun and Bai, Zechen and Chen, Yanzhe and Shou, Mike Zheng},
  journal={arXiv preprint arXiv:2603.02175},
  year={2026}
}
\bibliographystyle{icml2026}

\newpage
\appendix
\onecolumn
\section*{Appendix}

\section{One-step Sampling Analysis}




\subsection{The Comparisons of Different Inference Steps}
Video subtitle removal is a strongly constrained video editing task in which the input and target videos differ only slightly. Our experiments show that satisfactory results can be achieved with a limited number of sampling steps during inference. Therefore, we investigate the effect of the sampling step count on the performance of subtitle removal. As illustrated in Figure~\ref{fig:sample_steps}, setting the sampling steps to 1, 2, or 4 has minimal impact on the visual results. The configuration with $step=1$ effectively restores complex clothing textures.

\begin{figure*}[htbp]
    \centering
    \includegraphics[width=0.95\textwidth]{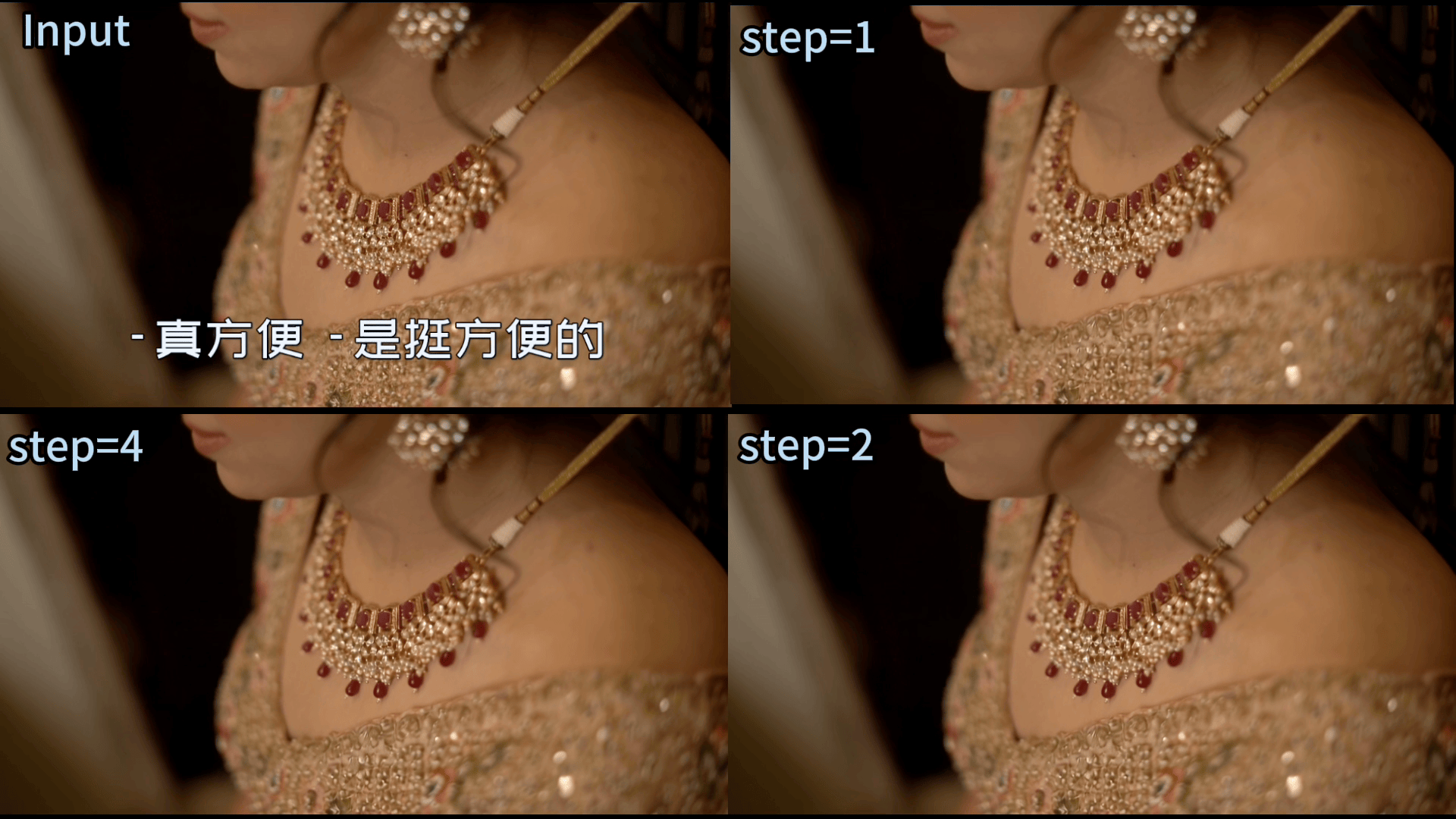}
    \caption{Comparison of visual effects under different sampling steps.}
    \label{fig:sample_steps}
\end{figure*}

We provide additional comparisons in Table~\ref{tab:video_edit_comparison}, including a 30-step SEDiT, Kiwi-Edit~\cite{kiwi-edit} (a general video editing model), and CLEAR~\cite{CLEAR} (a video subtitle removal model).

\begin{table}[htbp]
\centering
\caption{Quantitative comparison of video editing methods.}
\label{tab:video_edit_comparison}
\begin{tabular}{lcccc}
\toprule
Method & PSNR $\uparrow$ & SSIM $\uparrow$ & LPIPS $\downarrow$ & FVD $\downarrow$ \\
\midrule
SEDiT (1-step)      & \cellcolor{pink}31.5863 & \cellcolor{pink}0.8805 & \cellcolor{pink}0.0981 & \cellcolor{pink}24.0599 \\
SEDiT (4-step)      & 31.2961 & 0.8783 & 0.1001 & 26.7560 \\
SEDiT (30-step)     & 30.2953 & 0.8626 & 0.1072 & 31.7589 \\
Kiwi-Edit (1-step)  & 29.0176 & 0.8739 & 0.1038 & 34.3346 \\
Kiwi-Edit (30-step) & 28.6635 & 0.8725 & 0.1045 & 31.4693 \\
CLEAR (5-step) & 19.3893 & 0.7182 & 0.2379 & 167.6627 \\
\bottomrule
\end{tabular}
\end{table}

We observe that:
\begin{itemize}
    \item 30-step sampling does not improve quality; instead, accumulated errors degrade performance.
    \item Kiwi-Edit shows the same trend: 30-step is worse than 1-step except for FVD.
    \item This is consistent with our 1-step and 4-step results.
\end{itemize}

\subsection{Theoretical analysis: conditional rectified flow for subtitle removal}
In the subtitle removal task, the input consists of a reference video latent $\mathbf{z}_\text{ref}$, which itself contains subtitles, and a noisy target latent $\mathbf{z}_\text{noisy}$. Unlike unconditional diffusion models, we adopt a \textbf{Conditional Rectified Flow (CRF)} framework, where the generative process is modeled as learning a conditional velocity field. Formally, the velocity field is defined as
\begin{equation}
    v_{\theta }\left ( \mathbf{z},t \mid \mathbf{z}_\text{ref} \right ) \approx v^{*} \left ( \mathbf{z},t \mid \mathbf{z}_\text{ref} \right ) ,
\end{equation}
with $v^{*}$ denoting the true velocity induced by the conditional optimal transport path. The training objective minimizes the discrepancy between the predicted and true velocity fields:
\begin{equation}
    \min_{\theta }\mathbb{E}_{\mathbf{z},t}\left [ \left \| v_{\theta }\left ( \mathbf{z},t \mid \mathbf{z}_\text{ref} \right ) - v^{*}\left ( \mathbf{z},t \mid \mathbf{z}_\text{ref} \right )  \right \| ^2 \right ].
\end{equation}
The model must therefore learn to accurately identify and suppress subtitle regions while maintaining temporal consistency in the surrounding non-subtitle areas. This setting implies that the divergence between the target distribution (subtitle-free video) and the reference distribution (video with subtitles) is concentrated within a small masked region $\mathbf{M}$, while the remainder of the frame remains nearly identical. Consequently, the conditional optimal transport path is globally close to linear, requiring only localized corrections in $\mathbf{M}$.

According to the forward derivation of Rectified Flow, when the velocity field is approximately linear, and the perturbation is spatially limited, a single integration step suffices to approximate the true flow:
\begin{equation}
    \mathbf{z} = \mathbf{z}_\text{noisy} + v_{\theta }\left ( \mathbf{z}_\text{noisy},t \mid \mathbf{z}_\text{ref} \right ) \cdot \Delta t .
\end{equation}
This formulation demonstrates that, under the CRF framework, the noisy target latent can be directly mapped to the subtitle-free latent in one step, guided by the conditional velocity field. The theoretical implication is that \textbf{one-step denoising is sufficient and effective for subtitle removal}, as the task involves only localized distributional shifts while leveraging the strong contextual information encoded in the reference video.

\section{Proof of Lipschitz Continuity of the Conditional Optimal Transport Path under Video Subtitle Erasure}
\label{proof:1}
This section provides a formal justification that, under the localized-editing structure of subtitle removal, the conditional optimal transport (OT) map and its induced rectified flow velocity field are Lipschitz continuous with respect to the latent variable. This property underpins the theoretical feasibility of one-step sampling.

\subsection{Problem Setup}

Let the latent variable be decomposed into a subtitle region and a non-subtitle region:
\begin{equation}
   \mathbf{z} = (\mathbf{z}_{\mathbf{M}}, \mathbf{z}_{\neg \mathbf{M}}), 
\end{equation}
where $\mathbf{M}$ denotes the subtitle mask. The reference distribution $q(\mathbf{z}\mid \mathbf{z}_{\text{ref}})$ and the target distribution $p(\mathbf{z}\mid \mathbf{z}_{\text{ref}})$ differ only within $\mathbf{M}$.
The conditional optimal transport map is denoted by $T^*(\mathbf{z}\mid \mathbf{z}_{\text{ref}})$, and the corresponding conditional OT path is:
\begin{equation}
    \gamma^*(t,\mathbf{z}\mid \mathbf{z}_{\text{ref}}) = (1-t)\mathbf{z} + t\,T^*(\mathbf{z}\mid \mathbf{z}_{\text{ref}}),
\end{equation}
with velocity field:
\begin{equation}
   v^*(\mathbf{z},t\mid \mathbf{z}_{\text{ref}}) = T^*(\mathbf{z}\mid \mathbf{z}_{\text{ref}}) - \mathbf{z}. 
\end{equation}
\subsection{Structural Decomposition of the Conditional OT Map}
Let $T^{*}\left ( \cdot \mid \mathbf{z}_\text{ref} \right ) $ denote the conditional OT map that pushes the reference latent distribution $q(\mathbf{z} \mid \mathbf{z}_\text{ref})$ to the target distribution $p(\mathbf{z} \mid \mathbf{z}_\text{ref})$. Due to the localized nature of the distributional shift, the OT map decomposes as
\begin{equation}
    T^*(\mathbf{z}\mid \mathbf{z}_{\text{ref}}) 
= \big( T^*_{\mathbf{M}}(\mathbf{z}_{\mathbf{M}},\mathbf{z}_{\neg \mathbf{M}}\mid \mathbf{z}_{\text{ref}}),\; \mathbf{z}_{\neg \mathbf{M}} \big).
\end{equation}
i.e., the non-subtitle region is preserved exactly:
\begin{equation}
    T^*_{\neg \mathbf{M}}(\mathbf{z}_{\neg \mathbf{M}}) = \mathbf{z}_{\neg \mathbf{M}}.
\end{equation}
This reflects the fact that subtitle removal should not alter background content or motion outside the subtitle region.
\subsection{Lipschitz Assumption on Local Transport}
Because the subtitle region is small, the perturbation is localized, and the OT cost is quadratic in the latent space, it is natural to assume that the local OT map on $\mathbf{M}$ is Lipschitz continuous with respect to the full latent vector. 

\begin{assumption}\label{assump:1} 
\textbf{(Local Lipschitz Continuity).}  
There exists $L_{\mathbf{M}}>0$ such that for any 
\begin{equation}
    \mathbf{z}^{(1)} = (\mathbf{z}^{(1)}_{\mathbf{M}},\mathbf{z}^{(1)}_{\neg \mathbf{M}}), 
    \quad
    \mathbf{z}^{(2)} = (\mathbf{z}^{(2)}_{\mathbf{M}},\mathbf{z}^{(2)}_{\neg \mathbf{M}}),
\end{equation}
we have
\begin{equation}
\big\| T^*_{\mathbf{M}}(\mathbf{z}^{(1)}_{\mathbf{M}},\mathbf{z}^{(1)}_{\neg \mathbf{M}}\mid \mathbf{z}_{\text{ref}}) 
- T^*_{\mathbf{M}}(\mathbf{z}^{(2)}_{\mathbf{M}},\mathbf{z}^{(2)}_{\neg \mathbf{M}}\mid \mathbf{z}_{\text{ref}}) \big\|
\le L_{\mathbf{M}} \big\| \mathbf{z}^{(1)} - \mathbf{z}^{(2)} \big\|.
\label{eq:local-lip-en}
\end{equation}

\end{assumption}
\subsection{Lipschitz Continuity of the Conditional OT Map}

\begin{proposition}
Under Assumption~\ref{assump:1}, the conditional OT map $T^*$ is Lipschitz with constant $L_T = \sqrt{L_{\mathbf{M}}^2 + 1}$.
\end{proposition}
\textit{Proof.}  
For any $\mathbf{z}^{(1)},\mathbf{z}^{(2)}$:
\begin{equation}
    \begin{aligned}
        \big\| T^*(\mathbf{z}^{(1)}\mid \mathbf{z}_{\text{ref}}) 
        &- T^*(\mathbf{z}^{(2)}\mid \mathbf{z}_{\text{ref}}) \big\|^2 \\
        &= \big\| T^*_{\mathbf{M}}(\mathbf{z}^{(1)}_{\mathbf{M}},\mathbf{z}^{(1)}_{\neg \mathbf{M}}\mid \mathbf{z}_{\text{ref}}) 
        - T^*_{\mathbf{M}}(\mathbf{z}^{(2)}_{\mathbf{M}},\mathbf{z}^{(2)}_{\neg \mathbf{M}}\mid \mathbf{z}_{\text{ref}}) \big\|^2 \\
        &\quad + \big\| \mathbf{z}^{(1)}_{\neg \mathbf{M}} - \mathbf{z}^{(2)}_{\neg \mathbf{M}} \big\|^2.
    \end{aligned}
\end{equation}
Applying \eqref{eq:local-lip-en} and the inequality
\begin{equation}
    \|\mathbf{z}^{(1)}_{\neg \mathbf{M}} - \mathbf{z}^{(2)}_{\neg \mathbf{M}}\|
    \le \|\mathbf{z}^{(1)} - \mathbf{z}^{(2)}\|,
\end{equation}
we obtain
\begin{equation}
    \big\| T^*(\mathbf{z}^{(1)}\mid \mathbf{z}_{\text{ref}}) 
    - T^*(\mathbf{z}^{(2)}\mid \mathbf{z}_{\text{ref}}) \big\|
    \le \sqrt{L_{\mathbf{M}}^2 + 1}\, \big\| \mathbf{z}^{(1)} - \mathbf{z}^{(2)} \big\|.
\end{equation}

Thus, the conditional OT map is Lipschitz continuous with constant $L_T = \sqrt{L_{\mathbf{M}}^2 + 1}$. \hfill $\square$

\subsection{Lipschitz Continuity of the Velocity Field}
Since $v^*(\mathbf{z}) = T^*(\mathbf{z}) - \mathbf{z}$, we have:
\begin{equation}
    \|v^*(\mathbf{z}^{(1)}) - v^*(\mathbf{z}^{(2)})\|
\le (L_T + 1)\|\mathbf{z}^{(1)} - \mathbf{z}^{(2)}\|.
\end{equation}

\begin{figure}[!t]
    \centering
    \includegraphics[width=0.48\textwidth]{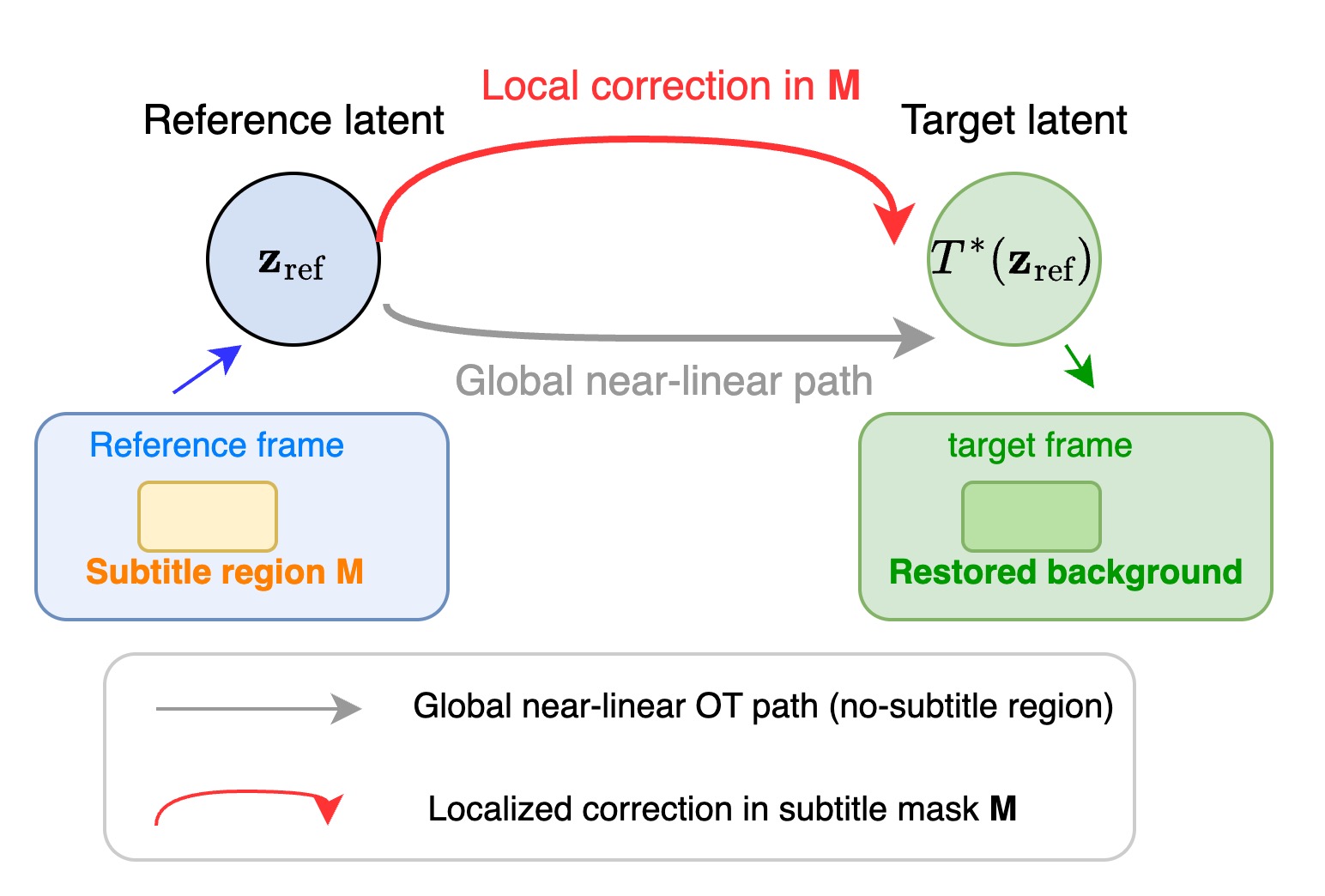}
    \caption{Schematic illustration of the conditional optimal transport (OT) path under localized subtitle removal. 
    The reference and target distributions differ only within a small masked region $\mathbf{M}$, while the remaining 
    content is preserved. As a result, the conditional OT path is globally near-linear and only requires localized 
    corrections within $\mathbf{M}$, leading to a Lipschitz-continuous velocity field and enabling accurate one-step sampling.}
    \label{fig:ot_localized}
\end{figure}

\subsection{Lipschitz Continuity of the Conditional OT Path}
The conditional OT path is defined as
\begin{equation}
    \gamma^*(t,\mathbf{z}\mid \mathbf{z}_{\text{ref}}) = (1-t)\,\mathbf{z} + t\,T^*(\mathbf{z}\mid \mathbf{z}_{\text{ref}}), \quad t \in [0,1].
\end{equation}
For any $\mathbf{z}^{(1)}, \mathbf{z}^{(2)}$
\begin{equation}
    \begin{aligned}
    \big\| \gamma^*(t,\mathbf{z}^{(1)}\mid \mathbf{z}_{\text{ref}}) 
    &- \gamma^*(t,\mathbf{z}^{(2)}\mid \mathbf{z}_{\text{ref}}) \big\| \\
    &\le (1-t)\big\| \mathbf{z}^{(1)} - \mathbf{z}^{(2)} \big\|
    + t L_T \big\| \mathbf{z}^{(1)} - \mathbf{z}^{(2)} \big\| \\
    &= \big( (1-t) + t L_T \big)\, \big\| \mathbf{z}^{(1)} - \mathbf{z}^{(2)} \big\| \\
    &\le \max\{1, L_T\}\, \big\| \mathbf{z}^{(1)} - \mathbf{z}^{(2)} \big\|.
    \end{aligned}
\end{equation}
Hence, the conditional OT path is Lipschitz continuous with constant $\max\{1, L_T\}$.

\subsection{Implication for One-Step Sampling}
In the subtitle erasure task, the mask $\mathbf{M}$ occupies only a small spatio-temporal region, and the background restoration is smooth in both space and time. Consequently, the local Lipschitz constant $L_{\mathbf{M}}$ is moderate, and thus $L_T$ and $\max\{1,L_T\}$ remain well controlled. As shown in Figure~\ref{fig:ot_localized}, the OT path is therefore nearly linear in the non-subtitle region and only mildly perturbed within the subtitle region. This regularity and near-linearity are precisely the structural reasons why one-step rectified flow can accurately approximate the transport dynamics in this task.

\section{Additional Analyses and Discussions}
\subsection{Spatial-temporal cues of subtitles}
Although our method does not use explicit masks, it localizes subtitles through spatial-temporal cues, which naturally distinguish subtitles from background content.

\begin{table*}
\caption{Subtitles and background follow distinct spatial-temporal patterns.}
\label{tab:spatial-temporal-cues}
\begin{tabular}{lccccc}
\Xhline{1.2pt}
\makecell{\textbf{Case}} & \makecell{\textbf{Subtitle}\\\textbf{Variation}} & \makecell{\textbf{Background}\\\textbf{Variation}} & \makecell{\textbf{Spatial-Temporal}\\\textbf{Pattern}} & \makecell{\textbf{Subtitle}\\\textbf{Localization}} & \makecell{\textbf{Typical}\\\textbf{Scenario}} \\
\Xhline{0.8pt}
\hline
\makecell{1.Subtitle changes, \\ background \\ nearly static} &
\makecell{\checkmark\ Significant \\ changes (e.g., \\ sentence \\ switch, \\ fade-in/out)} &
\makecell{\xmark\ Almost \\ unchanged} & 
\makecell{Subtitle variation $\gg $ \\ background variation} &
\makecell{\textbf{Easy to localize}} & 
\makecell{Static camera, \\ subtitles updating} \\
\hline
\makecell{2. Subtitle static, \\ background changes} & 
\makecell{\xmark\ Nearly \\ unchanged} &
\makecell{\checkmark\ Significant \\ changes \\ (camera \\ motion, object \\ movement, \\ illumination changes)} &
\makecell{Background variation \\ $\gg $ subtitle variation} &
\makecell{\textbf{Easy to localize}} &
\makecell{Subtitle \\ persists \\ while the \\ scene is \\ dynamic} \\
\hline
\makecell{3. Subtitle static, \\ background static \\ (extreme case)} &
\makecell{\xmark\ Unchanged} &
\makecell{\xmark\ Unchanged} &
\makecell{No spatial-temporal \\ cues $\to$ relies \\ purely on spatial cues} &
\makecell{Difficult (must infer \\ from subtitle \\ style/shape)} &
\makecell{Very short clips, \\ fully static \\ scenes} \\
\hline
\makecell{4. Subtitle and \\ background both \\ change} &
\makecell{\checkmark\ Changes} &
\makecell{\checkmark\ Changes} &
\makecell{Different change \\patterns: subtitles \\ are localized \& \\ structured; \\ background changes \\ are natural \& irregular} & 
\makecell{Locatable \\ (pattern‑based)} &
\makecell{Camera motion \\ + subtitle \\ transitions}\\
\Xhline{1.2pt}
\end{tabular}

\end{table*}

As summarized in Table~\ref{tab:spatial-temporal-cues}, the subtitles exhibit stable and structured behavior, while the background regions show natural and irregular temporal variations. Even when both regions change, their change patterns remain fundamentally different, providing reliable cues for localization without explicit supervision.

\subsection{The Impact of User Prompt}

\begin{figure*}[htbp]
    \centering
    \includegraphics[width=0.9\textwidth]{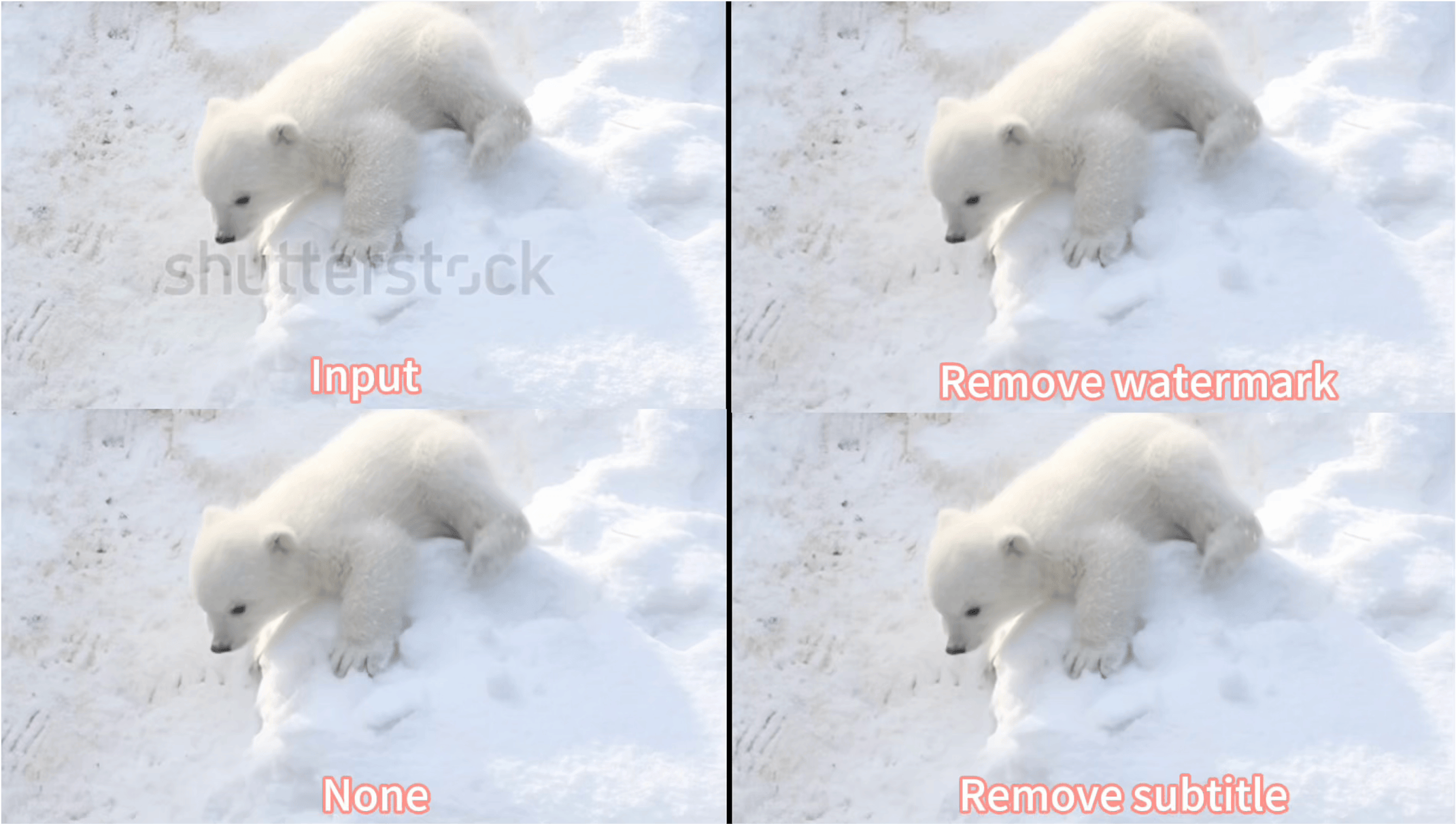}
    \caption{Results of watermark removal.}
    \label{fig:watermark}
\end{figure*}

In the Figure~\ref{fig:watermark}, ``Remove watermark'' corresponds to the user prompt: ``Please remove the \textbf{watermark} from the video while preserving the character appearance, background composition, and color style. Do not add any new elements.'' ``Remove subtitle'' indicates the use of a predefined prompt, while ``None'' means the prompt is an empty string (``"). The results suggest that the subtitle removal prompt has little impact on the output. This is because the training process was designed as a single-task setup. Since the prompt is fixed, the text encoder model does not need to be loaded during inference. Instead, we utilize precomputed prompt embeddings, which reduce memory consumption and improve inference efficiency.

\subsection{Positional Encoding Selection}
The motivation behind our 3D RoPE design stems from the frame-wise alignment nature of the video subtitle removal task. Therefore, we use identical positional encoding for reference and noisy latents: $\mathbf{u}_\text{ref} = \mathbf{u}_\text{noisy} = (f,h,w)$. In the LTX-Video model, the upper limit of $f$ is set to $20$, corresponding to a video length of $20 \times 8 + 1=161$ frames. If we adopt a shifted encoding scheme, such as $\mathbf{u}_\text{ref} = (f,h,w), \mathbf{u}_\text{noisy}=(f+\delta, h, w)$, then the maximum $f$ becomes $10$, limiting the reference video to $10 \times 8 + 1 = 81$ frames. This means a single chunk can only process up to $81$ frames, which is insufficient for our 121-frame setting in 720p resolution. Another alternative is to introduce an extra dimension to distinguish reference and noisy latents: $\mathbf{u}_\text{ref} = (0, f, h, w), \mathbf{u}_\text{noisy} = (1, f, h, w)$. However, this 4D RoPE has a different feature dimensionality from the original 3D RoPE, requiring retraining of the LTX-Video base model and significantly increasing computational cost. Therefore, we chose the most direct, effective, and semantically clear approach.

\subsection{The Visual Results for Video Style Transfer}
To evaluate the generalization capability of the SEDiT architecture and the effectiveness of user prompts, we train a video stylization model on the recently released video editing dataset Ditto-1M~\cite{ditto-1m}. As shown in Figure~\ref{fig:video_style_results}, the \textbf{user prompt} precisely guides the model to generate distinct visual styles. This demonstrates that our model architecture can be readily adapted to other video editing tasks. Unlike subtitle removal, stylization is a global transformation and requires $15$ sampling steps.
\begin{figure*}[!t]
    \centering
    \includegraphics[width=0.9\textwidth]{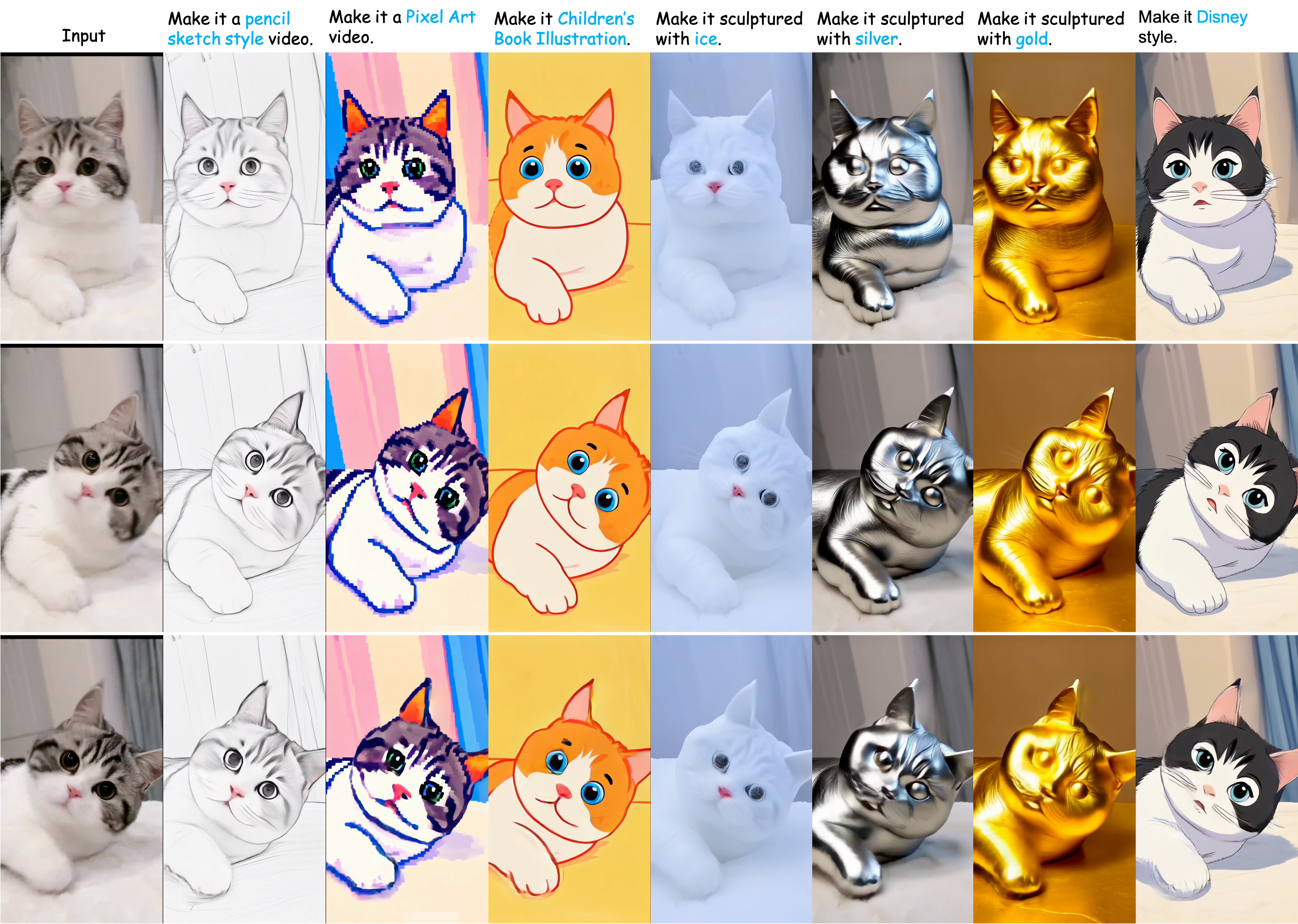}
    \caption{Visual Results for video style transfer.}
    \label{fig:video_style_results}
\end{figure*}

\subsection{Failure case of OCR}

As shown in Figure~\ref{fig:failure_case_ocr}, we apply the current state-of-the-art OCR algorithm, PaddleOCR-VL~\cite{paddleocr-vl}, to detect text in videos containing background characters. In addition to subtitle content, the algorithm also detects other textual elements present in the scene. As shown in the result of the second row, first column, PaddleOCR-VL fails to effectively recognize subtitles with transition effects. Such detection errors directly lead to the failure of mask-based subtitle removal methods. In contrast, our approach successfully removes these subtitles while preserving the original textual content in the scene.

\begin{figure*}[htbp]
    \centering
    \includegraphics[width=0.8\textwidth]{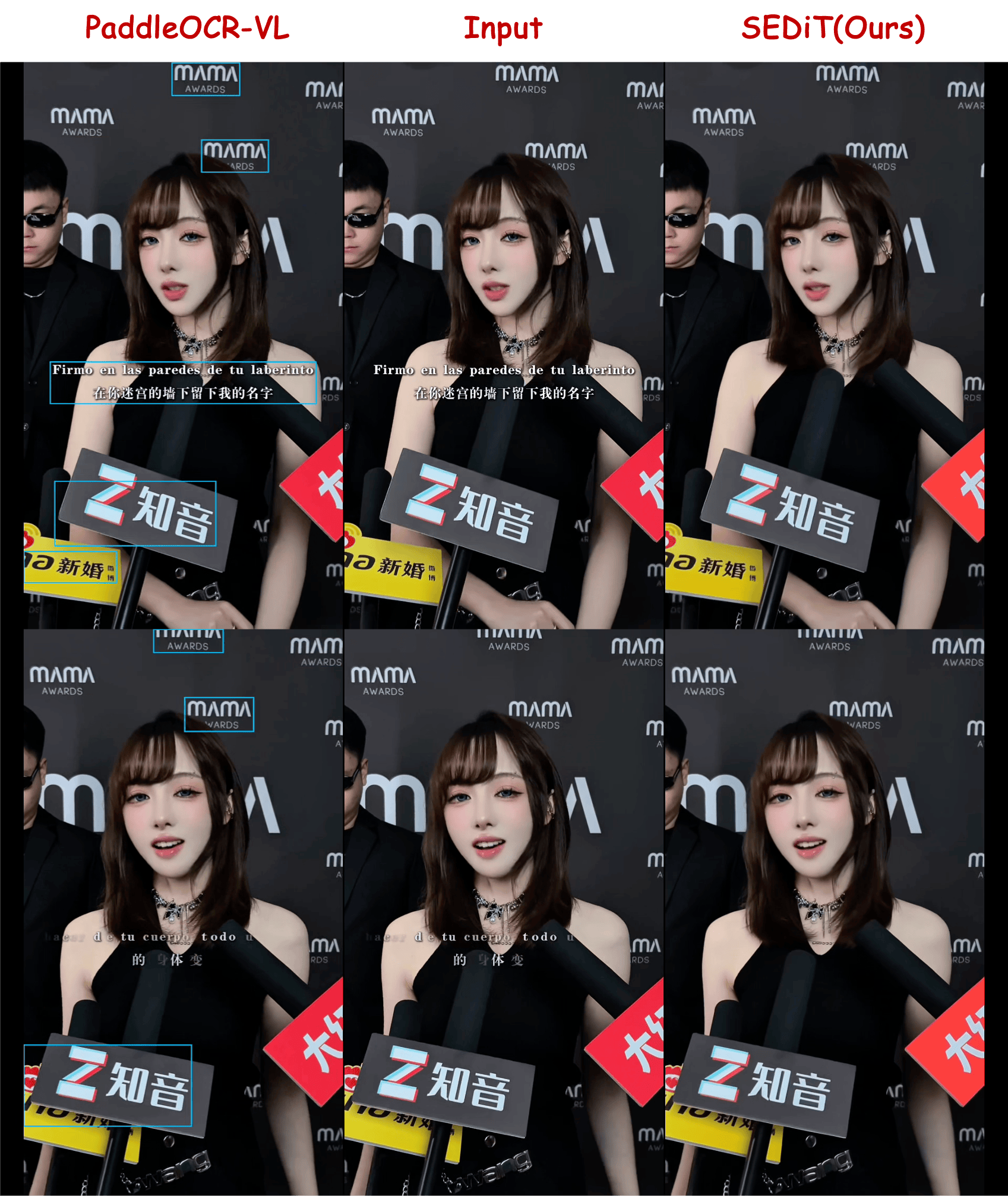}
    \caption{The OCR failure case in subtitle recognition. \emph{Best viewed when zoomed in.}}
    \label{fig:failure_case_ocr}
\end{figure*}

\section{More visual comparisons on VSR-Bench-400 dataset}
We further compared the visual performance of SEDiT with Propainter~\cite{propainter}, DiffuEraser~\cite{diffueraser}, and Minimax-Remover~\cite{minimax-remover} on the VSR-Bench-400 dataset. As shown in Figures~\ref{fig:vsr_bench_1},~\ref{fig:vsr_bench_2},~\ref{fig:vsr_bench_3}, Propainter tends to produce blurriness or artifacts when subtitles occupy a large proportion of the frame. DiffuEraser alleviates some artifacts, but still generates structurally inconsistent content. Among the mask-based approaches, Minimax-Remover achieves the best overall performance, though blurry regions remain. In general, our method strives to preserve as much of the original video content as possible, giving visually satisfactory results.

\begin{figure}[htbp]
    \centering
    \includegraphics[width=\textwidth]{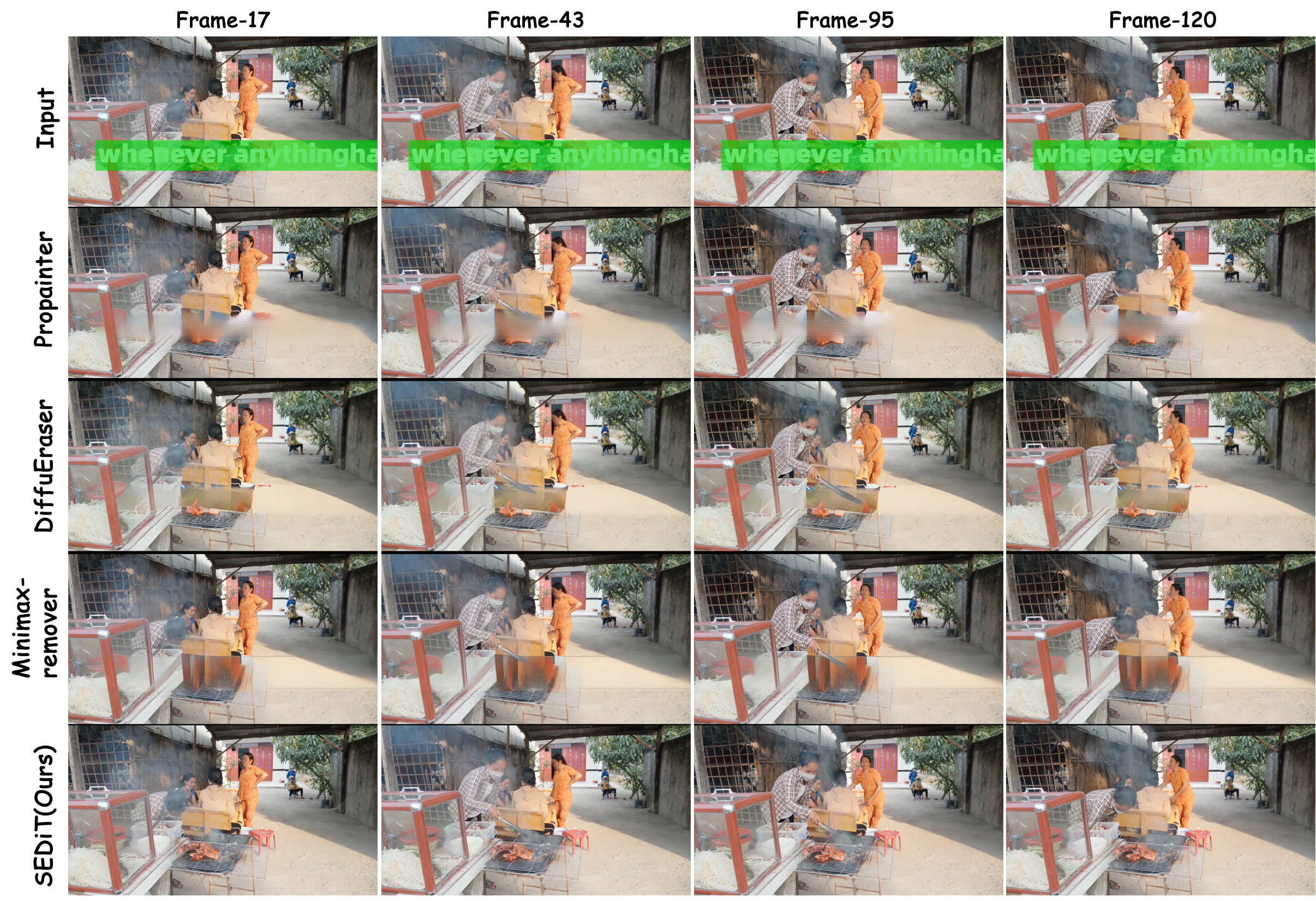}
    \caption{The visual comparison on VSR-Bench-400 dataset. The green highlighted region
            represents the subtitle mask, which is the ground-truth (GT) masks generated by filling the subtitle boundaries.}
    \label{fig:vsr_bench_1}
\end{figure}

\begin{figure}[htbp]
    \centering
    \includegraphics[width=\textwidth]{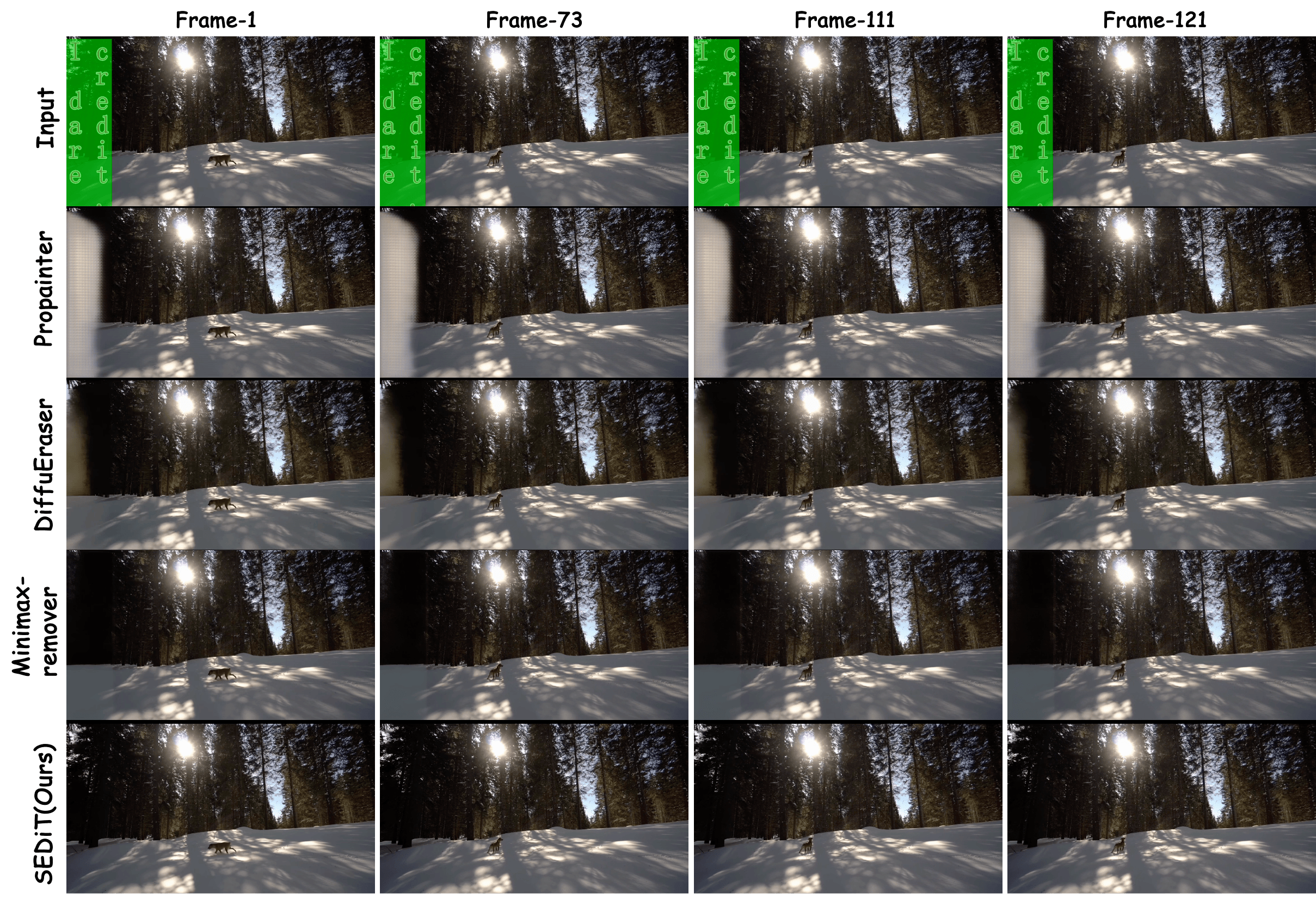}
    \caption{The visual comparison on VSR-Bench-400 dataset. The green highlighted region
            represents the subtitle mask, which is the ground-truth (GT) masks generated by filling the subtitle boundaries.}
    \label{fig:vsr_bench_2}
\end{figure}

\begin{figure}[htbp]
    \centering
    \includegraphics[width=\textwidth]{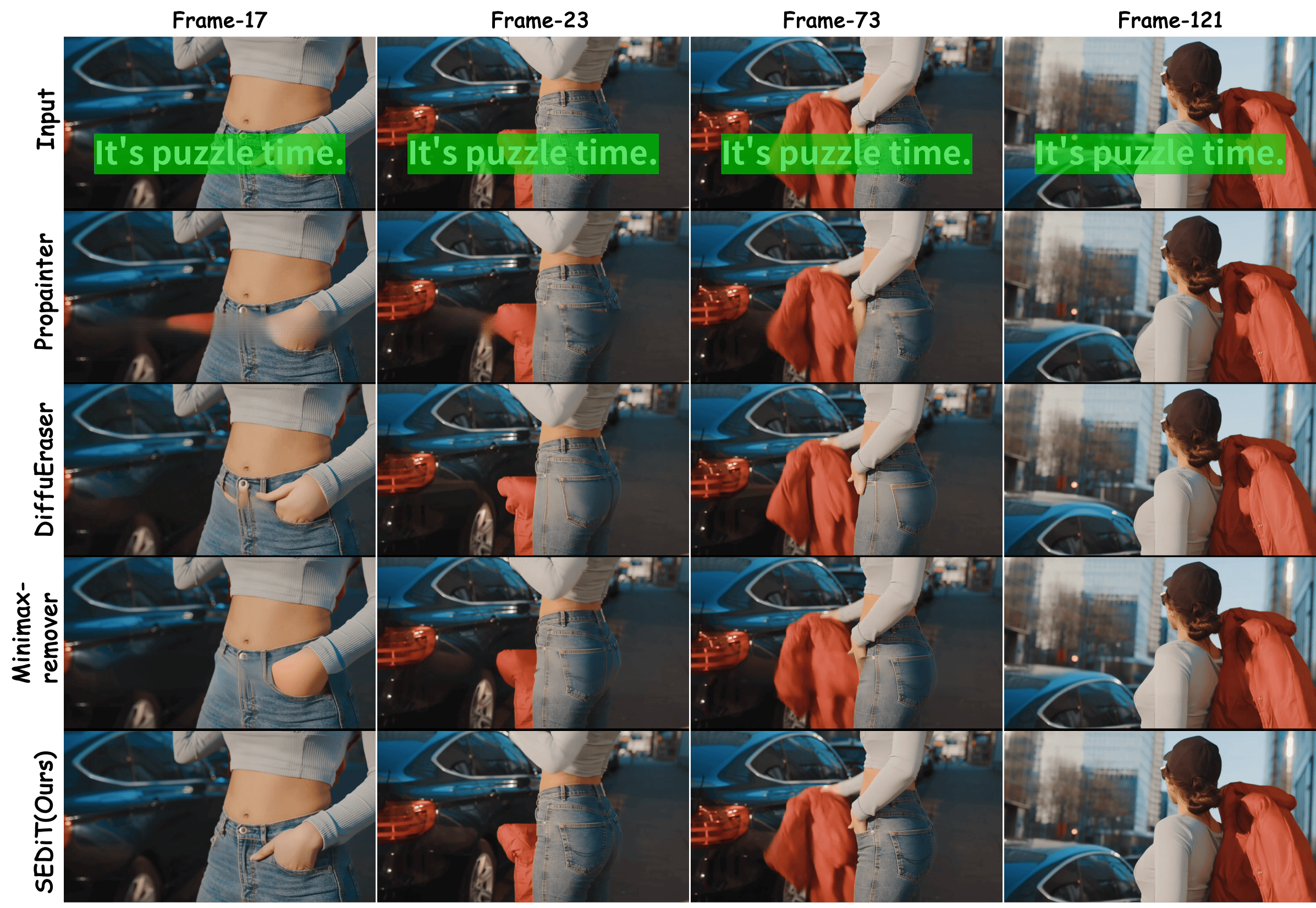}
    \caption{The visual comparison on VSR-Bench-400 dataset. The green highlighted region
            represents the subtitle mask, which is the ground-truth (GT) masks generated by filling the subtitle boundaries.}
    \label{fig:vsr_bench_3}
\end{figure}


\end{document}